\documentclass[twoside,11pt]{article}

\usepackage{blindtext}

\usepackage[abbrvbib, preprint]{jmlr2e}

\usepackage{physics}
\usepackage{float}
\usepackage[nameinlink, noabbrev]{cleveref}
\Crefname{corollary}{Corollary}{Corollaries}
\Crefname{lemma}{Lemma}{Lemmas}
\Crefname{figure}{Figure}{Figures}
\Crefname{definition}{Definition}{Definitions}
\Crefname{inequality}{inequality}{inequalities}
\Crefname{example}{Example}{Examples}

\usepackage{tikz}
\usetikzlibrary{decorations.pathmorphing}
\usepackage{subcaption}

\newcommand{\N}{\mathbb{N}}

\newcommand{\R}{\mathbb{R}}

\newcommand{\inv}{{^{-1}}}

\newcommand{\setst}{\mid}
\newcommand{\altsetst}{\,:\,}
\newcommand{\given}{\mid}
\newcommand{\altgiven}{;\,}
\newcommand{\iid}{\overset{iid}{\sim}}
\DeclareMathOperator*{\argmin}{\arg\,\min}

\newcommand{\tp}{^\top}
\DeclareMathOperator*{\E}{E}
\DeclareMathOperator*{\Var}{Var}
\renewcommand{\var}{\Var}

\newcommand{\D}{\mathcal{D}}
\renewcommand{\H}{\mathcal{H}}
\newcommand{\X}{\mathcal{X}}
\newcommand{\Y}{\mathcal{Y}}
\newcommand{\rhat}{\widehat{R}}
\newcommand{\thetahat}{\widehat{\theta}}

\DeclareMathOperator{\bel}{bel}
\DeclareMathOperator{\pl}{pl}
\newcommand{\Thetahat}{\widehat{\Theta}}
\newcommand{\varthetahat}{\widehat{\vartheta}}

\usepackage{lastpage}
\jmlrheading{23}{2022}{1-\pageref{LastPage}}{1/21; Revised 5/22}{9/22}{21-0000}{Neil Dey and Jonathan P. Williams}

\ShortHeadings{Valid Inference for Machine Learning Model Parameters}{Dey and Williams}
\firstpageno{1}

\begin{document}

\title{Valid Inference for Machine Learning Model Parameters}

\author{\name Neil Dey \email ndey3@ncsu.edu\\
\name Jonathan P. Williams \email jwilli27@ncsu.edu \\
       \addr Department of Statistics\\
       North Carolina State University\\
       Raleigh, NC 27607-6698, USA
}

\editor{My editor}

\maketitle

\begin{abstract}%   <- trailing '%' for backward compatibility of .sty file
The parameters of a machine learning model are typically learned by minimizing a loss function on a set of training data. However, this can come with the risk of overtraining; in order for the model to generalize well, it is of great importance that we are able to find the optimal parameter for the model on the entire population---not only on the given training sample. In this paper, we construct valid confidence sets for this optimal parameter of a machine learning model, which can be generated using only the training data without any knowledge of the population. We then show that studying the distribution of this confidence set allows us to assign a notion of confidence to arbitrary regions of the parameter space, and we demonstrate that this distribution can be well-approximated using bootstrapping techniques. 
\end{abstract}


\section{Introduction}
Many machine learning (ML) applications train a model by finding parameters that minimize a loss function (such as the zero-one loss or squared loss) on some training data. This approach generalizes well outside the training data if the average loss on the training data is sufficiently close to the expected loss (i.e. ``risk") for the entire population. Classical \textit{probably approximately correct} (PAC) learning theory  \citep[e.g.][]{valiant1984, kearns1994, mohri2018} provides theoretical rates at which the empirical risk function converges to the population risk in various settings, but provides no theoretical guarantees about estimation properties of the risk \textit{minimizer}.
In general, little work has been done to quantify the uncertainty in an empirical risk minimizer (ERM).

Given that an ERM is a function of the data, aleatory uncertainty is a natural quantification of the inherent randomness of the ERM: 
It is desirable to be able to construct confidence sets and hypothesis tests for risk minimizers of machine learning models that are ``valid" in the sense that they posses frequentist repeated sampling guarantees. The utility of validity in machine learning contexts is perhaps most clearly illustrated by conformal prediction algorithm, introduced by \citet{vovk2005}. Conformal prediction can take a point prediction method (for classification or regression problems) and yield valid prediction regions for new data, in the sense that the generated prediction region will contain the true label with any prespecified level of confidence. Thus, there exists a very useful validity guarantee on the prediction regions generated by conformal prediction-enhanced machine learning algorithms. Because this validity property is so desirable, conformal prediction has found use in a variety of applications, such as online regression forests \citep{vasiloudis2019}, QSAR modelling for drug development \citep{eklund2015}, convolutional neural networks for image classification \citep{matiz2019}, and various deep-learning architectures \citep{soundouss2020}, among others. However, conformal prediction cannot quantify uncertainty in the estimate of the ERM (i.e. one of the ``inputs" of an untrained ML model)---only on the outputs of the trained model.

We also note that in the context of the broader development in the literature of inferential approaches with finite-sample validity guarantees (not necessarily restricted to ERMs), papers on ``e-values'' and ``e-processes'' have been gaining prominence---most notably in the the work of \citet{grunwald2018}, \citet{grunwald2020b}, \citet{wasserman2020}, and \citet{ramdas2023}.  An e-value is a nonnegative statistic that is bounded in mean by $1$ under the null hypothesis, and an e-process is a nonnegative supermartingale such that the stopped process is always an e-value regardless of the choice of stopping time. These e-values and e-processes serve as ``safer" alternatives to traditional p-values as a measure of evidence, providing robustness to practices such as optional continuation of data collection as well as greater robustness to model misspecification. Related to these themes is the work of \cite{grunwald2022} on ``safe probability'' more generally as it relates to the foundations of statistics and formulations of Bayesian, fiducial, and imprecise probabilistic ideas; see also \cite{williams2023}.

Of the statistical literature that exists for uncertainty quantification of ERMs in particular, much of it falls under the umbrella of the more general framework of M-estimation \citep{huber1981}. The statistical properties of M-estimators are well studied \citep[e.g.][]{maronna2006, geer2000,boos2018}; however, the theory of M-estimation typically only yields asymptotic frequentist guarantees that require strong consistency conditions or distributional assumptions. Another approach is found in \citet{hudson2021}, which studies hypothesis testing for the risk minimizer using a nonparametric extension of the classical score test. Unfortunately, the theory from \citet{hudson2021} requires smoothness conditions on the risk function, and again only yields asymptotic frequentist guarantees.  This problem is also studied in \citet{cella2022} via the framework of generalized inferential models, but is once again only able to provide approximate validity due to the need to appeal to asymptotic theory. Bayesian approaches to uncertainty quantification of the risk minimizer, such as the penalized exponentially tilted empirical likelihood posterior of \citet{tang2021} or the Gibbs-posterior inference approaches of \citet{martin2013} and \citet{bhattacharya2022}, are similarly only capable of giving asymptotic coverage guarantees. 

Another set of Bayesian approaches is characterized by the use of Gibbs posteriors (also sometimes called generalized Bayes posteriors) as studied in \citet{grunwald2017}, \citet{grunwald2020}, and \citet{wu2023}, among others. Once again, however, these approaches either ignore frequentist guarantees altogether, or fail to provably yield finite-sample coverage guarantees. Two approaches that do yield the guarantees of interest for risk minimizers include \citet{waudbysmith2023} and \citet{dey2024}; in fact, they go further and provide \textit{anytime-valid} coverage, in the sense that the coverage guarantees hold regardless of the stopping rule used for the data collection process, due to their use of e-processes for inference. However, both of these approaches require strong assumptions: \citet{waudbysmith2023} is only able to handle the $L^2$ loss for bounded random variables, whereas \citet{dey2024} requires a ``strong central condition" to hold for the data, and its efficiency guarantees require strictly stronger assumptions than those that will be used in this manuscript. It thus remains an open question: Can we obtain efficient, valid confidence sets and hypothesis tests for these risk minimizers under minimal assumptions on the data-generating process? In this paper, we answer this question in the affirmative for all machine learning models satisfying a relatively weak \textit{uniform convergence} assumption.

The primary insight of this paper is illustrated in \Cref{fig:basic_insight}: When performing inference on a risk minimizer $\theta_0$, we typically estimate it via an ERM $\thetahat_S$ that depends on a random sample $S$. We then quantify our uncertainty in the ERM by examining a neighborhood $\Thetahat_S^\varepsilon$ around $\thetahat_S$ of some size $\varepsilon$; ideally, this neighborhood would contain $\theta_0$, but it is typically unclear how large $\varepsilon$ needs to be to do so. The primary reason doing so is difficult is that the empirical risk may behave very differently from the population risk. A simple example of this appears when using a machine learning model that always overfits to match the data perfectly, so that the empirical risk on a sample is always zero despite the population risk being rather high---this may cause the distance between the ERM and risk minimizer to vary unpredictably across samples. However, if we can ensure that the empirical risk is, with high probability, sufficiently close to the population risk in a uniform manner over the parameter space $\Theta$, then this issue is solved and it is possible find an $\varepsilon$ such that $\Thetahat_S^\varepsilon$ contains $\theta_0$ with high probability. 
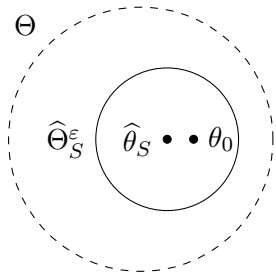
\begin{figure}
    \centering
\begin{tikzpicture}
    \def\radius{50}
    \def\radiuspt{\radius pt}
    % \Theta
    \node[circle, draw, dashed, minimum size={2*\radiuspt}, label=above left:$\Theta$] (fullspace) at (0,0){};
    % \theta_0
    \node[circle, draw, fill=black, inner sep=0pt, minimum size = 3pt, label=right:$\theta_0$](theta_0) at ({2*\radiuspt/5}, 0){};
    % \thetahat_S
    \node[circle, draw, fill=black, inner sep=0pt, minimum size = 3pt, label=left:$\thetahat_S$](thetahat) at ({\radiuspt/5}, 0){};
    %\Thetahat^\varepsilon_S
    \node[circle, draw, minimum size={\radiuspt+4}, label=left:$\Thetahat_S^\varepsilon$, ](ci)at (thetahat){};
\end{tikzpicture}
    \caption{To quantify our uncertainty in the ERM $\thetahat_S$, we look at the size of the $\varepsilon$-neighborhood $\Thetahat_S^\varepsilon$ around $\thetahat_S$ that is intended to cover the risk minimizer $\theta_0$.}
    \label{fig:basic_insight}
\end{figure}

The contributions of our paper are as follows:
\begin{itemize}
    \item We provide a new framework for statistical inference that can be used to provide uncertainty quantification for a parameter of interest without knowledge of the data-generating distribution. 
    \item We use this framework to create confidence sets that contain the ``optimal" (i.e. risk minimizing) parameter of a machine learning model with any pre-specified level of confidence.
    \item We demonstrate that knowledge of the theoretical distribution of valid confidence sets can be used to determine the location of risk minimizer. 
    \item We draw on the ideas of imprecise probability theory \citep{dempster1967,shafer1976} to study this theoretical distribution of confidence sets. This is in part inspired by the growing use of imprecise probability in the study of both empirical risk minimization and statistical inference \citep[e.g.][]{kusunoki2021,hullermeier2014,jacob2021,gong2021}.
    \item We show that under weak conditions, bootstrapping can be used to estimate the distribution of these confidence sets. This allows us to assign a notion of confidence to arbitrary regions of the parameter space, even when the data-generating distribution is completely unknown. 
    \item We illustrate the utility of our framework in both synthetic and real data examples\footnote{The code to reproduce the experiments presented in this paper is available at \url{https://github.com/neil-dey/valid-inference-ml-estimators}}.
\end{itemize}

The remainder of the paper is structured as follows: In \Cref{sec:learning_problem}, we give context for the problem that machine learning aims to solve and background on classical machine learning theory that will be referenced throughout the paper. In \Cref{sec:framework}, we propose our new inferential framework and demonstrate how to construct valid confidence sets for the risk minimizer of a machine learning model. In \Cref{sec:distbn}, we examine how to assign confidence to arbitrary regions of the parameter space by using the theory of random sets and imprecise probability to study the distribution of our valid confidence sets. \Cref{sec:IMs} briefly discusses the related work of generalized Inferential Models and gives a comparison between the two approaches. We finish with concluding remarks and directions for future work in \Cref{sec:conclusion}.

\section{The Supervised Learning Problem}\label{sec:learning_problem}
For a machine learning task, one receives a sequence of examples $X_1, \ldots, X_m$ from an example domain $\mathcal{X}$ and corresponding labels $Y_1, \ldots, Y_m$ from a label space $\mathcal{Y}$; the task of interest is being able to predict the label $Y$ associated with a new, unseen example $X$. To do so, the practitioner selects a \textit{hypothesis class} $\mathcal{H}$ consisting of functions (referred to as hypotheses) $\mathcal{X}\rightarrow\mathcal{Y}$ and a loss function $L:\mathcal{Y}\times\mathcal{Y}\rightarrow\R^+$. The supervised learning problem is to choose the hypothesis that minimizes the expected loss over the (unknown) data-generating distribution.

A natural question is to ask when the supervised learning problem is tractable to solve, and what it formally means to solve it in the first place. One notion is given by the probably approximately correct (PAC) learning framework, introduced by \citet{valiant1984}. In this framework, we suppose that there exists a \textit{true concept} $c\in\H$ such that for any example $X$, its corresponding label is given by $c(X)$. The randomness in the training sample\footnote{In this paper, we define samples to consist of independent and identically distributed observations.} $S = \{(X_1, c(X_1)), \ldots, (X_m, c(X_m))\}$ typically prevents a guarantee that $c$ can be found exactly. Hence, the PAC-learning framework claims that a procedure ``successfully learns" if it can find, with high probability, a hypothesis ``close enough" to $c$ without needing an excessively large sample size. That is, PAC-learning requires that there exist an algorithm that probably finds a hypothesis that is approximately correct. Formally, we have the following definition:
\begin{definition}\label[definition]{def:pac}
A hypothesis class $\H$ is PAC-learnable with respect to a fixed loss function $L$ if there exists an learning process\footnote{formally, an algorithm $A:\bigcup_{i=1}^\infty (\X\times\Y)^i \rightarrow \H$, where the domain is the set of all samples of finite size} and a polynomial $p$ such that for every distribution $\D$ over $\mathcal{X}$, every $\varepsilon>0$, every $\delta>0$, and every concept $c\in\H$,
\begin{equation*}
    \Pr_{S\sim\D^m}\qty[\E_{X\sim \D}\qty[L(\widehat{h}_S(X), c(X))] \leq \varepsilon] \geq 1 - \delta
\end{equation*}
whenever the sample size $m$ is at least $p(\varepsilon\inv, \delta\inv)$, where $\widehat{h}_S$ is the output of the learning process when given the sample $S$ as an input.
\end{definition}

The PAC-learning framework was later was later extended to the agnostic PAC-learning model introduced by \citet{kearns1994}, generalizing to no longer require that there exist a true concept $c$ that generates the labels. Instead, samples $S = ((X_1, Y_1), \ldots, (X_m, Y_m))$ can be generated from any distribution $\D$ over $\mathcal{X}\times\mathcal{Y}$. Although there need not be a true concept in $\H$, there still exists some lower bound on how badly the hypotheses in $\H$ perform; this lower bound now acts as the target to aim for. Thus, the agnostic PAC-learning framework defines a procedure to ``successfully learn" if it can probably find a hypothesis that is approximately as good as the best that $\H$ can possibly perform. Formally:
\begin{definition}
Let $\H$ be a hypothesis class and $L$ a loss function. The risk of a hypothesis $h\in\H$ on a distribution $\D$ over $\mathcal{X}\times\mathcal{Y}$ is the expected loss of $h$:
\begin{equation*}
    R(h) = \E_{(X, Y)\sim \D}[L(h(X), Y)].
\end{equation*}
\end{definition}
\begin{definition}\label[definition]{def:agpac}
A hypothesis class $\H$ is agnostically PAC-learnable with respect to a fixed loss function $L$ if there exists a learning process and a polynomial $p$ such that for every distribution $\D$ over $\mathcal{X}\times\mathcal{Y}$, every $\varepsilon>0$, and every $\delta>0$, 
\begin{equation*}
    \Pr_{S\sim\D^m}\qty[R(\widehat{h}_S) \leq \inf_{h\in \H} R(h) + \varepsilon] \geq 1 - \delta
\end{equation*}
whenever $m \geq p(\varepsilon\inv, \delta\inv)$, where $\widehat{h}_S$ is the output of the learning process  when given the sample $S$ as an input.
\end{definition}

The PAC-learning and agnostic PAC-learning frameworks thus give criteria to determine when the supervised learning problem is tractable. However, they do not tell us how to solve the problem. Although we wish to find $h\in\H$ that minimizes $R(h)$, we do not know the data-generating distribution $\D$ to directly minimize $R$; hence, practitioners will instead minimize the empirical risk on the given sample $S$:
\begin{equation*}
    \rhat_S(h) = \frac{1}{m}\sum_{i=1}^m L(h(X_i), Y_i).
\end{equation*}
The hypothesis that minimizes $\rhat_S$ is then used as an approximation to the hypothesis that minimizes $R$; if one can guarantee that $\rhat_S$ is approximately the same as $R$ for large enough sample sizes, this yields an alternative criterion for the existence of a solution to the supervised learning problem. This idea is formalized by the notion of the Glivenko-Cantelli property:
\begin{definition}\label[definition]{def:sugc}
A hypothesis class $\H$ is a strong uniform Glivenko-Cantelli class with respect to a fixed loss function if for every $\varepsilon > 0$,
\begin{equation*}
    \lim_{n\rightarrow\infty}\sup_{\D}\Pr_{S}\qty[\sup_{m \geq n}\sup_{h\in\H}\qty|\rhat_S(h) - R(h)| > \varepsilon] = 0,
\end{equation*}
where the outermost supremum is understood to be over all possible distributions over $\X\times\Y$, and $S$ is a sample of size $m$ from $\D$. 
\end{definition}

\begin{quote}
\begin{remark}\normalfont{
All these notions of learning turn out to be related. Consider the class\footnote{We assume in this discussion that $L\circ\H$ is image-admissible Suslin; c.f. \citet[page 101]{dudley1984}}
\begin{equation*}
    L\circ \H = \{(x, y)\mapsto L(h(x), y) \setst h\in\H\}.
\end{equation*}
If $\H$ is a binary hypothesis class and $L$ is the zero-one loss function, then the works of \citet{vapnik1971} and \citet{assound1989} (see also \citealt{dudley1991} for an alternative, unified proof) give that $\H$ is a strong uniform Glivenko-Cantelli class with respect to $L$ if and only if $L\circ\H$ has finite VC dimension. Furthermore, \citet{vapnik1971} and \citet{blumer1989} together show that in this setting, $\H$ with respect to $L$ is agnostically PAC-learnable if and only if it is PAC-learnable if and only if it has finite VC dimension; that is to say that all of these notions are equivalent for binary hypothesis classes. This result was later generalized to real-valued hypothesis classes $\H$ and general loss functions $L$ by \citet{alon1997}, stating that $\H$ is a strong uniform Glivenko-Cantelli class with respect to $L$ if and only if $L\circ\H$ has finite $\gamma$-fat shattering dimension at all scales $\gamma > 0$, which in turn implies agnostic PAC-learnability.
}\end{remark}
\end{quote}

\section{An Inferential Framework for Machine Learning}\label{sec:framework}
We propose the following machine learning (ML) framework for inference: Data $(X_i, Y_i)_{i=1}^m$ are generated i.i.d. from an unknown \textit{data-generating} distribution $\D$ over the sample space $\X\times\Y$. We then fit a \textit{working model} $Y = h(X\altgiven \theta)$, where the set of candidate \textit{hypotheses} $\H = \{x\mapsto h(x\altgiven \theta) \setst \theta\in\Theta\}$ is known. In contexts where the model is not presented any ``inputs" to learn labels from, we take the convention\footnote{One may expect the convention $\X = \varnothing$ to be more natural, but the absence of examples is not the same as an empty example space.%: 
% Since $\H\subseteq \Y^{\X\times\Theta}$, if $\X = \varnothing$ then $\H \subseteq \Y^{\varnothing\times\Theta} = \Y^{\varnothing} = \varnothing$ so $\H = \varnothing$ and there is nothing to learn at all.
}
that $\X = \{\varnothing\}$. In an ML context, $\theta$ can be thought of as the vector of parameters that is to be ``learned" during the training of the model. As an example, if $\H$ were a class of neural networks, $\theta$ would be a vector of weights and biases corresponding to a particular trained model. We wish to find $\theta_0$ that best fits the generating distribution, in the sense that
\begin{equation*}
	\theta_0 := \argmin_{\theta\in\Theta} R(\theta) \equiv  \argmin_{\theta\in\Theta}\E_{(X, Y)\sim \D}\qty[L(h(X\altgiven \theta), Y)] 
\end{equation*}
for some fixed loss function $L:\Y\times\Y\rightarrow\R$. Such a $\theta_0$ is referred to as a \textit{risk minimizer}. Together, we call the pair $(\H, L)$ the ML model.

It is typical to estimate $\theta_0$ via an estimator $\thetahat_S$ found by empirical risk minimization:
\begin{equation*}
    \thetahat_S := \argmin_{\theta\in\Theta} \rhat_S(\theta) \equiv \argmin_{\theta\in\Theta}\frac{1}{m}\sum_{i=1}^m L(h(X_i\altgiven\theta), Y_i)
\end{equation*}
where $S = ((X_1, Y_1), \ldots, (X_m, Y_m)) \sim \D^m$ is the observed sample. We call such a $\thetahat_S$ an \textit{empirical risk minimizer} (ERM). A general property of ``learnable" ML models relating the risk and empirical risk is the \textit{uniform convergence property}:
\begin{definition}\label[definition]{def:ucp}
An ML model $(\H, L)$ has the uniform convergence property with respect to the data-generating distribution $\D$ over $\X\times \Y$ if there exists a function $f:\R^+\times\R^+\rightarrow \R$ such that for any $\varepsilon > 0$ and $\alpha > 0$, if $m \geq f(\varepsilon, \alpha)$ then 
\begin{equation*}
    \Pr_{S\sim\D^m}\qty[\sup_{\theta\in\Theta}|R(\theta) - \rhat_S(\theta)| \leq \varepsilon] \geq 1-\alpha.
\end{equation*}
We call any such $f$ a \textit{witness} to the uniform convergence property.
\end{definition}

% \begin{theorem}
% If $(\H, L)$ is such that $\H$ is a strong uniform Glivenko-Cantelli class, then it has the uniform convergence property.
% \end{theorem}
% \begin{proof}
% Basically just write out the definitions.
% \end{proof}

\begin{quote}
\begin{remark}\normalfont{
Notice that \Cref{def:ucp} is a much weaker assumption than that of being a strong uniform Glivenko-Cantelli class (\Cref{def:sugc}); it follows that all image-admissible Suslin ML models with finite VC dimension or $\gamma$-fat-shattering dimension for all $\gamma > 0$ have the uniform convergence property with respect to \textit{any} distribution $\D$. We thus see that a broad class of models of interest fulfill this requirement.
}\end{remark}
\end{quote}

\begin{definition}\label[definition]{def:ucf}
Let $(\H, L)$ have the uniform convergence property, and let $\mathcal{W}$ be the set of all corresponding witnesses.
The uniform convergence function of $(\H, L)$ is defined to be
\begin{equation*}
    f(\varepsilon, \alpha) = \inf_{w\in\mathcal{W}} \lceil w(\varepsilon, \alpha)\rceil,
\end{equation*}
where $\lceil\,\cdot\,\rceil$ is the ceiling function.
\end{definition}
In other words, the uniform convergence function of $(\H, L)$ is the smallest integer-valued function that witnesses the uniform convergence property. 
\begin{quote}
\begin{remark}\normalfont{
    For image-admissible Suslin classes of finite VC or $\gamma$-fat-shattering dimension, upper bounds on the worst-possible asymptotic behavior of the uniform convergence function are well-known, since this is simply the rate of convergence for the strong Glivenko-Cantelli class. For finite samples, these bounds can be obtained via tools such as the Rademacher complexity or covering numbers.
}\end{remark}
\end{quote}

Though it is typical to only examine ERMs for prediction problems, we find that for inference problems, it is useful to look at parameters that are \textit{almost} ERMs:
\begin{definition}
Let $(\H, L)$ be an ML model and $S$ a given sample. The set of $\varepsilon$-almost ERMs ($\varepsilon$-AERMs) is defined to be \begin{equation*}
	\Thetahat_S^\varepsilon = \{\theta\in\Theta \setst \rhat_S(\theta) \leq \inf_{\vartheta\in\Theta} \rhat_S(\vartheta) + \varepsilon\}.
\end{equation*}
\end{definition}
It turns out that this set of $\varepsilon$-AERMs acts as a valid confidence set for $\theta_0$:

\begin{theorem}\label{thm:eseps_ci}
Let $(\H, L)$ have uniform convergence function $f$. Suppose that the risk minimizer $\theta_0$ exists. Then $\Thetahat_S^\varepsilon$ is a $1-\alpha$ level confidence set for $\theta_0$ if $m\geq f(\varepsilon/2, \alpha)$.
\end{theorem}
\begin{proof}
We have by definition that 
\begin{equation*}
    \Pr[\theta_0 \in \Thetahat_S^\varepsilon] = \Pr[\rhat_S(\theta_0) \leq \inf_{\vartheta\in\Theta}\rhat_S(\vartheta) + \varepsilon].
\end{equation*}
Next, note by definition of the infimum that for every $\zeta>0$ there exists $\theta_\zeta\in\Theta$ such that $\rhat_S(\theta_\zeta) - \inf_{\vartheta\in\Theta}\rhat_S(\vartheta) \leq \zeta$. 
Furthermore, we have by the uniform convergence property that if $m \geq f(\varepsilon/2, \alpha)$, then
\begin{equation*}
	\Pr[\sup_{\vartheta\in\Theta}|\rhat_S(\vartheta) - R(\vartheta)| \leq \frac{\varepsilon}{2}] \geq 1-\alpha.
\end{equation*}
Now if $\sup_{\vartheta\in\Theta}|\rhat_S(\vartheta) - R(\vartheta)| \leq \varepsilon/2$, then for every $\theta\in\Theta,\enspace|\rhat_S(\theta_0) - R(\theta_0)| + |\rhat_S(\theta) - R(\theta)| \leq \varepsilon$; this is true in particular for every $\theta_\zeta$, so
\begin{equation}\label[inequality]{eq:ucpzeta}
	\Pr[|\rhat_S(\theta_0) - R(\theta_0)| + \inf_{\zeta>0}|\rhat_S(\theta_\zeta) - R(\theta_\zeta)| \leq \varepsilon] \geq 1-\alpha.
\end{equation}
Next, we have for every $\zeta > 0$ that
\begin{flalign*}
	&&&\rhat_S(\theta_0) - \inf_{\vartheta\in\Theta}\rhat_S(\vartheta) &\\
		&&&= \rhat_S(\theta_0) - R(\theta_0) + R(\theta_0) - R(\theta_\zeta) + R(\theta_\zeta) - \rhat_S(\theta_\zeta) + \rhat_S(\theta_\zeta) - \inf_{\vartheta\in\Theta}\rhat_S(\vartheta) &\\
		&&&= \qty[\rhat_S(\theta_0) - R(\theta_0)] - \qty[R(\theta_\zeta) - R(\theta_0)] + \qty[R(\theta_\zeta) - \rhat_S(\theta_\zeta)] + \qty[\rhat_S(\theta_\zeta) - \inf_{\vartheta\in\Theta}\rhat_S(\vartheta)] &(\text{Regrouping})\\
		&&&\leq |\rhat_S(\theta_0) - R(\theta_0)| - 0 + |R(\theta_\zeta) - \rhat_S(\theta_\zeta)| + \qty[\rhat_S(\theta_\zeta) - \inf_{\vartheta\in\Theta}\rhat_S(\vartheta)] & (\theta_0\text{ minimizes } R)\\
		&&&\leq |\rhat_S(\theta_0) - R(\theta_0)| + |R(\theta_\zeta) - \rhat_S(\theta_\zeta)|  + \zeta. & (\text{Definition of }\theta_\zeta)
\end{flalign*}
Taking the infimum over $\zeta > 0$, we thus have that
\begin{equation}\label[inequality]{eq:theta0empriskbound}
	\rhat_S(\theta_0) - \inf_{\theta\in\Theta}\rhat_S(\theta)  \leq |\rhat_S(\theta_0) - R(\theta_0)| + \inf_{\zeta >0} |R(\theta_\zeta) - \rhat_S(\theta_\zeta)|.
\end{equation}
Hence, we have by combining inequalities \eqref{eq:ucpzeta} and \eqref{eq:theta0empriskbound} that
\begin{equation*}
	\Pr[\rhat_S(\theta_0) - \inf_{\vartheta\in\Theta}\rhat_S(\vartheta) \leq \varepsilon] \geq \Pr[|\rhat_S(\theta_0) - R(\theta_0)| + \inf_{\zeta > 0}|R(\theta_\zeta) - \rhat_S(\theta_\zeta)| \leq \varepsilon] \geq 1-\alpha.
\end{equation*}
as desired.
\end{proof}
The above theorem has the following intuition: The ERM should be close to the risk minimizer with high probability due to the uniform convergence property, so looking at some sufficiently large $\varepsilon$-neighborhood of the ERM ought to capture the risk minimizer with high probability.

While it is often the case that $\theta_0 = \argmin_{\theta\in\Theta} R(\theta)$ exists, this is not always the case. If $\Theta$ is compact (with respect to the topology induced by the pseudometric $d(x, y) = |R(x) - R(y)|$), then it is necessarily the case that the risk minimizer $\theta_0$ exists. However, when $\Theta$ is not compact, $\theta_0$ may not exist---there only exists a neighborhood around $\theta_0$ on the boundary of $\Theta$. Thus, we instead aim to cover the neighborhood
    \[\Theta_0^{\delta} = \{\theta\in\Theta \setst R(\theta) \leq \inf_{\vartheta\in\Theta}R(\vartheta) + \delta\}\] 
for some $\delta \geq 0$. Note that the case $\delta = 0$ is of interest when $\theta_0$ does exist (since then $\Theta_0^0 = \{\theta_0\}$) or is not unique (as is often the case in models such as neural networks, which present multiple global minima). \Cref{fig:paramspacenotcompact} illustrates this phenomenon.
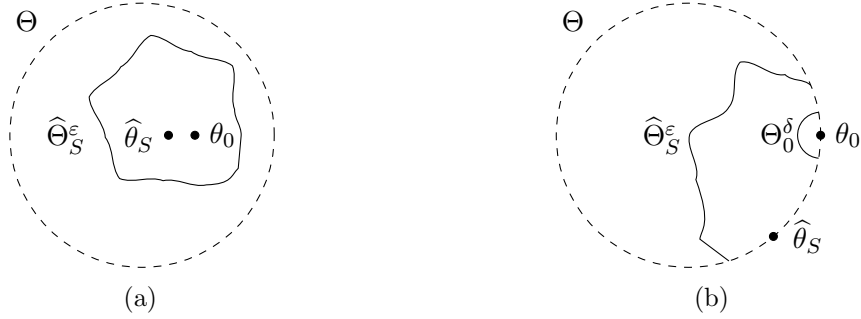
\begin{figure}[tbp]
\centering
\begin{subfigure}[t]{0.49\textwidth}
\centering
\begin{tikzpicture}
    \def\radius{50}
    \def\radiuspt{\radius pt}
    % \Theta
    \node[circle, draw, dashed, minimum size={2*\radiuspt}, label=above left:$\Theta$] (fullspace) at (0,0){};
    % \theta_0
    \node[circle, draw, fill=black, inner sep=0pt, minimum size = 3pt, label=right:$\theta_0$](theta_0) at ({2*\radiuspt/5}, 0){};
    % \thetahat_S
    \node[circle, draw, fill=black, inner sep=0pt, minimum size = 3pt, label=left:$\thetahat_S$](thetahat) at ({\radiuspt/5}, 0){};
    %\Thetahat^\varepsilon_S
    \node[circle, draw, minimum size={\radiuspt+5}, label=left:$\Thetahat_S^\varepsilon$, decorate, decoration={snake, segment length={10pt*pi}}](ci)at (thetahat){};
\end{tikzpicture}
\caption{}% \caption{The risk minimizer exists within the parameter space and is covered by $\Thetahat^\varepsilon_S$, which contains the ERM.}
\end{subfigure}
% Don't erase this space
%
\begin{subfigure}[t]{0.49\textwidth}
\centering
\begin{tikzpicture}
    \def\radius{50}
    \def\radiuspt{\radius pt}
    %\Theta
    \node[circle, draw, dashed, minimum size=2*\radiuspt, label=above left:$\Theta$] (fullspace) at (0,0){};
    
    %\theta_0
    \node[circle, draw, fill=black, inner sep=0pt, minimum size = 3pt, label=right:$\theta_0$](theta_0) at (\radiuspt, 0){};
    
    \def\angle{10};
    %\Theta_0^\delta
    \draw ({\radiuspt*cos(\angle)}, {\radiuspt*sin(\angle)}) arc[start angle={90+\angle/2}, end angle={270-\angle/2}, radius={\radiuspt*sin(\angle)/sin((180-\angle)/2)}] node[midway] {$\Theta_0^\delta\hspace{15pt}$};
    
    %\Thetahat_S^\varepsilon
    \def\angle{-25}
    \def\nbhdrad{40}
    \def\b{acos(\nbhdrad/(2*\radius))}
    %\draw ({\radius*cos(\angle)}, {\radius*sin(\angle)}) arc[start angle={90+\angle/2}, end angle={270-\angle/2}, radius={\radius*sin(\angle)/sin((180-\angle)/2)}] node[midway] {$\Thetahat_S^\varepsilon\hspace{20pt}$};
    
    %\thetahat_S
    \node[circle, draw, fill=black, inner sep=0pt, minimum size = 3pt, label=right:$\,\thetahat_S$](thetahat) at ({\radiuspt*cos(2*\angle)}, {\radiuspt*sin(2*\angle)}){};
    
    %\Thetahat_S^\varepsilon but a circular arc
    %\draw ({\radiuspt*cos(180-2*\b+\angle)}, {\radiuspt*sin(180-2*\b+\angle)}) arc[start angle={180+\angle-\b}, end angle={180+\angle+\b}, radius={\nbhdrad pt}] node[midway] {$\Thetahat_S^\varepsilon\hspace{20pt}$};
    
    %\Thetahat_S^\varepsilon
    \draw[decorate, decoration={snake, segment length={10pt*pi}, amplitude=4pt}] ({\radiuspt*cos(\angle + 180-2*\b)}, {\radiuspt*sin(\angle + 180-2*\b)}) to [out=135, in=180] ({{\radiuspt*cos(\angle - (180-2*\b))}}, {{\radiuspt*sin(\angle - (180-2*\b))}}) node[midway]{$\Thetahat_S^\varepsilon\hspace{20pt}$};
\end{tikzpicture}
\caption{}%\caption{Neither the risk minimizer nor the ERM exists within the parameter space; $\Thetahat_S^\varepsilon$ instead contains the closed ball $\Theta_0^\delta$}
\end{subfigure}
%\caption{Illustration of how the risk minimizer $\theta_0$ and/or the ERM $\thetahat_S$ may fail to exist when the parameter space $\Theta$ is not compact, and the desired behavior of $\Thetahat^\varepsilon_S$.}
\caption{Illustrations of the desired behavior of $\Thetahat_S^\varepsilon$ when the parameter space $\Theta$ is not compact. In (a), the risk minimizer $\theta_0$ and the ERM $\thetahat_S$ both exist; $\Thetahat_S^\varepsilon$ thus includes $\theta_0$. In (b), neither $\theta_0$ nor $\thetahat$ exist in the parameter space; $\Thetahat_S^\varepsilon$ instead contains the closed ball $\Theta_0^\delta$.}
\label{fig:paramspacenotcompact}
\end{figure}

The following theorem demonstrates that $\Thetahat_S^\varepsilon$ does indeed remain a valid confidence set for any such $\Theta_0^\delta$ (so long as $\varepsilon$ is chosen large enough).

\begin{theorem}\label{thm:eseps_bigcs}
Let $(\H, L)$ have uniform convergence function $f$. Then $\Thetahat_S^{\varepsilon}$ is a $1-\alpha$ level confidence set for $\Theta_0^\delta$ if $\delta\leq\varepsilon$ and $m \geq f((\varepsilon - \delta)/2, \alpha)$.
\end{theorem}
\begin{proof}
We have that
\begin{equation*}
    \Pr\qty[\Theta_0^\delta \subseteq \Thetahat_S^\varepsilon] = \Pr\qty[\sup_{\theta_0 \in \Theta_0^\delta}\rhat_S(\theta_0) \leq \inf_{\vartheta\in\Theta} \rhat_S(\vartheta) + \varepsilon].
\end{equation*}
As in the proof of \Cref{thm:eseps_ci}, we have by the definition of the infimum that for every $\zeta > 0$ there exists $\theta_\zeta \in \Theta$ such that $\rhat_S(\theta_\zeta) - \inf_{\vartheta\in\Theta}\rhat_S(\vartheta) \leq \zeta$. Furthermore, by definition of the supremum we have that for every $\eta > 0$ there exists $\theta_\eta \in \Theta_0^\delta$ such that $\sup_{\theta_0\in\Theta_0^\delta}\rhat_S(\theta_0) - \rhat_S(\theta_\eta) \leq \eta$. Then similarly to the argument in \Cref{thm:eseps_ci}, we have by the uniform convergence property that
\begin{equation}\label[inequality]{eq:etazetaucp}
    \Pr\qty[\inf_{\eta > 0}|\rhat_S(\theta_\eta) R(\theta_\eta)| + \inf_{\zeta > 0}|\rhat_S(\theta_\zeta) - \rhat(\theta_\zeta)| \leq \varepsilon - \delta] \geq 1-\alpha.
\end{equation}
Then for every $\eta > 0$ and $\zeta > 0$,
\begin{align*}
    &\sup_{\theta_0\in\Theta_0^{\delta}}\rhat_S(\theta_0) - \inf_{\vartheta\in\Theta} \rhat_S(\vartheta) \\
    &=\sup_{\theta_0\in\Theta_0^{\delta}}\rhat_S(\theta_0) - \rhat_S(\theta_\eta) + \rhat_S(\theta_\eta) - \rhat_S(\theta_\zeta) + \rhat_S(\theta_\zeta) - \inf_{\vartheta\in\Theta} \rhat_S(\vartheta) \\
    &\leq \eta + \zeta + \rhat_S(\theta_\eta) - \rhat_S(\theta_\zeta) \\
    &= \eta + \zeta + \rhat_S(\theta_\eta) - R(\theta_\eta) + R(\theta_\eta) - R(\theta_\zeta) + R(\theta_\zeta) - \rhat_S(\theta_\zeta) \\
    &\leq \eta + \zeta + |\rhat_S(\theta_\eta) - R(\theta_\eta)| + R(\theta_\eta) - R(\theta_\zeta) + |R(\theta_\zeta) - \rhat_S(\theta_\zeta)|
\end{align*}
Now recall that since $\theta_\eta \in \Theta_0^\delta$, we have that $R(\theta_\eta) \leq \inf_{\vartheta \in \Theta} R(\vartheta) + \delta$:
\begin{align*}
    &\sup_{\theta_0\in\Theta_0^{\delta}}\rhat_S(\theta_0) - \inf_{\vartheta\in\Theta} \rhat_S(\vartheta) \\
    &\leq \eta + \zeta + |\rhat_S(\theta_\eta) - R(\theta_\eta)| + |R(\theta_\zeta) - \rhat_S(\theta_\zeta)| + \inf_{\vartheta \in \Theta} R(\vartheta) + \delta - R(\theta_\zeta) \\
    &\leq \eta + \zeta + |\rhat_S(\theta_\eta) - R(\theta_\eta)| + |R(\theta_\zeta) - \rhat_S(\theta_\zeta)| + \delta.
\end{align*}
Taking the infimum over $\eta > 0$ and $\zeta > 0$, we thus find that
\begin{equation}\label[inequality]{eq:etazetadeltabound}
    \sup_{\theta_0\in\Theta_0^{\delta}}\rhat_S(\theta_0) - \inf_{\vartheta\in\Theta} \rhat_S(\vartheta) \leq \inf_{\eta > 0}|\rhat_S(\theta_\eta) R(\theta_\eta)| + \inf_{\zeta > 0}|\rhat_S(\theta_\zeta) - \rhat(\theta_\zeta)| + \delta.
\end{equation}
Hence, combining inequalities \eqref{eq:etazetaucp} and \eqref{eq:etazetadeltabound} yields that
\begin{align*}
    &\Pr[\sup_{\theta_0\in\Theta_0^{\delta}}\rhat_S(\theta_0) - \inf_{\vartheta\in\Theta} \rhat_S(\vartheta) \leq \varepsilon] \\
    &\geq \Pr\qty[\inf_{\eta > 0}|\rhat_S(\theta_\eta) R(\theta_\eta)| + \inf_{\zeta > 0}|\rhat_S(\theta_\zeta) - \rhat(\theta_\zeta)| \leq \varepsilon - \delta]\\ 
    &\geq 1-\alpha
\end{align*}
as desired.
\end{proof}

We note that \Cref{thm:eseps_ci,thm:eseps_bigcs} on their own may appear weak, as they present no upper bound on the size of $\Thetahat_S^\varepsilon$; indeed, simply choosing $\Thetahat_S^\infty = \Theta$ trivially acts as a $1-\alpha$ level confidence set but has no statistical power. The following theorem and corollary analyze the power of our proposed confidence sets by demonstrating that they are bounded to a reasonable size and shrink to contain only $\Theta_0^0$ as the sample size grows:

\begin{theorem}
Let $(\H, L)$ have uniform convergence function $f$. If $m \geq f(\varepsilon/2, \alpha)$, then 
\begin{equation*}
    \Pr_{S\sim\D^m}[\Thetahat^{\varepsilon}_S \subseteq \Theta_0^{2\varepsilon}] \geq 1-\alpha
\end{equation*}
Thus, our confidence set is bounded as $\Theta_0^0 \subseteq \Thetahat_S^\varepsilon \subseteq \Theta_0^{2\varepsilon}$ with probability at least $1-2\alpha$.
\end{theorem}
\begin{proof}
We want to show that for every $\theta\not\in\Theta_0^{2\varepsilon}$, $\rhat_S(\theta) > \inf_{\vartheta\in\Theta} \rhat_S(\vartheta) + \varepsilon$ with probability at least $1-\alpha$, or equivalently that
\begin{equation*}
    \Pr_{S\sim\D^m}\qty[\inf_{\theta\not\in\Theta_0^{2\varepsilon}}\rhat_S(\theta) - \inf_{\vartheta\in\Theta} \rhat_S(\vartheta) >  \varepsilon] \geq 1-\alpha
\end{equation*}
As in the previous theorems, for every $\gamma > 0$ there exists $\theta_\gamma \in \Theta$ such that $R(\theta_\gamma) \leq \inf_{\vartheta} R(\vartheta) + \gamma$. Hence, for every $\gamma > 0$:
\begin{align*}
& \inf_{\theta\not\in\Theta_0^{2\varepsilon}}\rhat_S(\theta) - \inf_{\vartheta\in\Theta} \rhat_S(\vartheta) \\
&= \inf_{\theta\not\in\Theta_0^{2\varepsilon}}\qty[\rhat_S(\theta) - R(\theta) + R(\theta)] -\inf_{\vartheta\in\Theta} R(\theta) +\inf_{\vartheta\in\Theta} R(\theta) - \inf_{\vartheta\in\Theta} \rhat_S(\vartheta) \\
&\geq \inf_{\theta\not\in\Theta_0^{2\varepsilon}}\qty[\rhat_S(\theta) - R(\theta)] + \qty[\inf_{\theta\not\in\Theta_0^{2\varepsilon}}R(\theta) - \inf_{\vartheta\in\Theta} R(\theta)] +\inf_{\vartheta\in\Theta} R(\theta) - \inf_{\vartheta\in\Theta} \rhat_S(\vartheta) \\
&\geq \inf_{\theta\not\in\Theta_0^{2\varepsilon}}\qty[\rhat_S(\theta) - R(\theta)] + 2\varepsilon +\inf_{\vartheta\in\Theta} R(\theta) - \inf_{\vartheta\in\Theta} \rhat_S(\vartheta) \\
&= \inf_{\theta\not\in\Theta_0^{2\varepsilon}}\qty[\rhat_S(\theta) - R(\theta)] + 2\varepsilon +\qty[\inf_{\vartheta\in\Theta} R(\theta) - R(\theta_\gamma)] + \qty[R(\theta_\gamma) - \rhat_S(\theta_\gamma)] + \qty[\rhat_S(\theta_\gamma)  - \inf_{\vartheta\in\Theta} \rhat_S(\vartheta)] \\
&\geq \inf_{\theta\not\in\Theta_0^{2\varepsilon}}\qty[\rhat_S(\theta) - R(\theta)] + 2\varepsilon - \gamma + \qty[R(\theta_\gamma) - \rhat_S(\theta_\gamma)] + 0
\end{align*}
Taking the supremum over $\gamma > 0$ yields
\begin{equation*}
    \inf_{\theta\not\in\Theta_0^{2\varepsilon}}\rhat_S(\theta) - \inf_{\vartheta\in\Theta} \rhat_S(\vartheta) \geq \inf_{\theta\not\in\Theta_0^{2\varepsilon}}\qty[\rhat_S(\theta) - R(\theta)] + \sup_{\gamma > 0} \qty[R(\theta_\gamma) - \rhat_S(\theta_\gamma)] + 2\varepsilon.
\end{equation*}
By similar arguments as the previous theorems, uniform convergence yields that for $m \geq f(\varepsilon/2, \alpha)$,
\begin{equation*}
    \Pr[\inf_{\theta\not\in\Theta_0^{2\varepsilon}}\qty[\rhat_S(\theta) - R(\theta)] + \sup_{\gamma > 0} \qty[R(\theta_\gamma) - \rhat_S(\theta_\gamma)] \geq -\varepsilon] \geq 1-\alpha,
\end{equation*}
from which the result follows.
\end{proof}
\begin{corollary}
Given a uniform convergence function $f$, define $\varepsilon(m, \alpha) = \inf\{\varepsilon \mid m\geq f(\varepsilon/2, \alpha)\}$. The confidence set $\Thetahat_S^{\varepsilon(m, \alpha)}$ converges in probability to $\Theta_0^0$, in the sense that for every $\zeta > 0$,
\begin{equation*}
    \lim_{m\rightarrow \infty} \Pr[\Thetahat_S^{\varepsilon(m, \alpha)} \subseteq \Theta_0^\zeta] = 1
\end{equation*}
\end{corollary}
\begin{proof}
Suppose $\eta$ is small enough so that $\eta < \alpha$. Then $\varepsilon(m, \alpha) \leq \varepsilon(m, \eta)$ for every $m$. Furthermore, it is straightforward to check that $\varepsilon(m, \eta)\rightarrow 0$ as $m\rightarrow\infty$, so for every $\zeta > 0$, $2\varepsilon(m, \eta) < \zeta$ for large enough $m$.  Hence, for every $\zeta>0$, we have for large enough $m$ that
\begin{equation*}
    \Pr[\Thetahat_S^{\varepsilon(m, \alpha)} \subseteq \Theta_0^\zeta] 
    \geq \Pr[\Thetahat_S^{\varepsilon(m, \eta)} \subseteq \Theta_0^\zeta] \\
    \geq \Pr[\Thetahat_S^{\varepsilon(m, \eta)} \subseteq \Theta_0^{2\varepsilon(m, \eta)}] \\
    \geq 1 - \eta.
\end{equation*}
Since $\eta>0$ can be made arbitrarily small, the result follows.
\end{proof}

\subsection{Examples}
We now present several examples in which we apply \Cref{thm:eseps_ci,thm:eseps_bigcs} to find valid confidence sets for a variety of parameters of interest.

\begin{example}\label[example]{ex:ber}\normalfont{
Suppose that we have a sample from a Bernoulli distribution and we wish to determine which of $0$ and $1$ is more likely in the population. This can be modelled as $\X = \{\varnothing\}$, $\Y = \{0, 1\}$, $\H = \{x \mapsto \theta \setst \theta \in \{0, 1\}\}$, and $L(y, y') = |y - y'|$ is the zero-one loss. Given this choice of loss function, the risk minimizer is $\theta_0 = 0$ if $p < 0.5$ and $\theta_0 = 1$ if $p > 0.5$; for $p=0.5$ the risk minimizer is not unique, so we must capture the set $\Theta_0^0 = \{0, 1\}$. 

We now compute the uniform convergence function for this model. The set of all distributions over $\X\times \Y$ is simply $\{\operatorname{Bernoulli}(p) \setst 0\leq p \leq 1\}$. For a fixed $p$, we have that
\begin{equation*}
    R(\theta) = \Pr_{Y\sim \operatorname{Ber}(p)}[\theta \neq Y] = (1-p)^\theta p^{1-\theta}.
\end{equation*}
We thus have that the left-hand side of \Cref{def:ucp} is given by
{\allowdisplaybreaks
\begin{align*}
    &\Pr\qty[\sup_{\theta = 0, 1}|R(\theta) - \rhat_S(\theta)| \leq \varepsilon] \\
    &= \Pr_{\vb*{Y}\sim\operatorname{Ber}(p)^m}\qty[\sup_{\theta = 0, 1}\qty| (1-p)^\theta p^{1-\theta} - \frac{1}{m}\sum_{i=1}^m I(Y_i \neq \theta)| \leq \varepsilon] \\
    &= \Pr\qty[\max\left\{\qty|p - \frac{1}{m}\sum_{i=1}^m I(Y_i \neq 0)|, \qty|1 - p - \frac{1}{m}\sum_{i=1}^m I(Y_i \neq 1)|\right\} \leq \varepsilon] \\
    &= \Pr\qty[\max\left\{\qty|p - \frac{1}{m}\sum_{i=1}^m Y_i|, \qty|-p + \frac{1}{m}\sum_{i=1}^m Y_i|\right\} \leq \varepsilon] \\
    &= \Pr\qty[\qty|\frac{1}{m}\sum_{i=1}^m Y_i - p| \leq \varepsilon] \\
    &= \Pr\qty[-m\varepsilon \leq \operatorname{Binom}(m, p) - mp \leq m\varepsilon] \\
    &= \sum_{i = \lceil m(p-\varepsilon)\rceil}^{\lfloor m(p+\varepsilon)\rfloor} \binom{m}{i}p^i(1-p)^{m-i}
\end{align*}
}
Setting the infimum of this quantity over $p\in[0, 1]$ to be at most $1 - \alpha$, we can numerically solve for $m$ or $\varepsilon$ to find valid confidence sets for the risk-minimizer $\theta_0$. 

Note that due to the discreteness of Bernoulli data, it is often impossible to obtain a coverage of exactly $1-\alpha$, so the confidence set will be conservative in general; this is not an issue present for continuous data.

%\Cref{fig:bernoulli_coverage} illustrates the coverage of valid confidence sets $\Thetahat_S^{\varepsilon}$ at $\alpha = 0.05$ as the data-generating distribution $\operatorname{Bernoulli}(p)$ varies for different sample sizes. We see that at all sample sizes, the coverage is always at least the expected $0.95$, as claimed by our theorems. We see that $\Thetahat_S^\varepsilon$ is typically quite conservative; as one would expect, it is only near $\operatorname{Bernoulli}(0.5)$ that the coverage starts dropping to the nominal level. It is also notable that $\Thetahat_S^\varepsilon$ appears to be more conservative at smaller sample sizes. %However, this is merely an artifact of the uniform convergence function having discrete jumps in $\varepsilon$; were it continuous, we would expect $\Thetahat_S^\varepsilon$ to attain a coverage of exactly $0.95$ on the worst-possible data-generating distribution regardless of sample size.

% \begin{figure}[tbp]
%     \centering
%     \includegraphics[width =0.75\textwidth]{./images/bernoulli_coverage.png}
%     \caption{The coverage of the valid set $\Thetahat_S^\varepsilon$ at different data-generating distributions $\operatorname{Bernoulli}(p)$ across different sample sizes. Here, $\varepsilon$ was chosen to be valid at significance level $0.05$.}
%     \label{fig:bernoulli_coverage}
% \end{figure}
}\end{example}

\begin{example}\normalfont{
Suppose we wish to find the $\tau$ quantile of a random variable with distribution $\D$. We can model this via $\X = \{\varnothing\}$, $\Y = \R$, $\H = \{x\mapsto \theta \setst \theta \in \Theta\}$, and $L(y, y') = (y'-y)(\tau - I(y' < y))$; the risk minimizer is precisely the $\tau$ quantile of the distribution. We have that
%\begin{align*}
%    R(\theta) 
%        &= \E[L(\theta, Y)] \\
%        &= \E[(Y - \theta)(\tau - I(Y < \theta))] \\
%        &= (\E[Y] - \theta)\tau - \E[(Y - \theta)I(Y < \theta)] \\
%        &= (\E[Y] - \theta)\tau - \E[(Y-\theta) \given Y < \theta] \cdot \Pr[Y < \theta]
%\end{align*}
\begin{equation*}
	R(\theta) = (\E[Y] - \theta)\tau - \E[(Y-\theta) \given Y < \theta] \cdot \Pr[Y < \theta]
\end{equation*}
On the other hand,
\begin{equation*}
    \rhat_S(\theta) 
    = \frac{1}{m}\sum_{i=1}^m L(\theta, y)
    = \qty(\overline{y} - \theta)\tau - \frac{1}{m}\sum_{y_i < \theta}(y_i - \theta),
\end{equation*}
where $\overline{y}$ is the sample mean $\sum_{i=1}^m y_i/m$. If $Y$ is a continuous random variable, then $R(\theta) - \rhat_S(\theta)$ is continuous in $\theta$, and so if $\Theta$ is compact, then for some $\widetilde{\theta}$ a function of $y_1, \ldots, y_m$:
\begin{align*}
    &\Pr_{S\sim\D^m}\qty[\sup_{\theta}|R(\theta) - \rhat_S(\theta)| \leq \varepsilon]\\
    &= \Pr_{S\sim\D^m}\qty[\qty|\E[Y]\tau - \E[(Y-\widetilde{\theta}) \mid Y < \widetilde{\theta}] \cdot \Pr[Y < \widetilde{\theta}] - \overline{y}\tau + \frac{1}{m}\sum_{y_i < \widetilde{\theta}}(y_i - \widetilde{\theta})|\leq \varepsilon] \\
    &\geq 1 - \frac{1}{\varepsilon^2}\var\qty[\overline{y}\tau - \frac{1}{m}\sum_{y_i < \widetilde{\theta}}(y_i - \widetilde{\theta})] \\
    &= 1 - \frac{1}{m\varepsilon^2}\var\qty[\tau Y - (Y - \widetilde{\theta})\cdot I(Y < \widetilde{\theta})].
\end{align*}
It suffices for this to be at least $1-\alpha$ for all $\widetilde{\theta}$ in the parameter space. Thus, solving for $m$, a uniform convergence function for quantile estimation is given by
\begin{equation*}
    f(\varepsilon, \alpha) = \frac{1}{\alpha\varepsilon^2}\sup_{\theta\in\Theta}\var\qty[\tau Y - (Y - \theta)\cdot I(Y < \theta)].
\end{equation*}
We see that we only require bounds on the first and second moments and conditional moments of the distribution in question to construct valid confidence sets for the quantile.

Of course, in the absence of bounds on the first and second moments, such a uniform convergence function is useless. However, a uniform convergence function may still be obtained using classical PAC-Learning theory which is distribution independent; by rescaling the loss function to be bounded in $[0, 1]$, Theorem 17.1 of \citet{anthony1999} yields the bound
\begin{equation*}
    \Pr_{S\sim\D^m}\qty[\sup_{\theta}|R(\theta) - \rhat_S(\theta)| \leq \varepsilon] \geq 1 - 4N_1(\varepsilon/16, L\circ \mathcal{H}, 2m)\exp(-\varepsilon^2m/32)
\end{equation*}
where $N_1(\varepsilon/16, L\circ\mathcal{H}, 2m)$ denotes the empirical $\ell^1$ covering number of $L\circ\mathcal{H}$ with balls of radius $\varepsilon/16$ and sample size $2m$. This no longer depends on the distribution of the data, but comes at the cost of efficiency, resulting in larger confidence sets than those that make use of moment bounds.
}
\end{example}

\begin{example}\label[example]{ex:lasso}{\normalfont
Consider LASSO estimation (without intercept):
\begin{equation*}
    \widehat{\beta}_{LASSO} = \argmin_{\beta\altsetst \norm{\beta}_1 \leq t} \frac{1}{m}\sum_{i=1}^m (y_i - x_i\tp\beta)^2
\end{equation*}
where $t$ is a regularization parameter typically chosen via cross-validation. We can consider the corresponding ML model $\X = \R^p$, $\Y = \R$, $\H = \{x\mapsto x\tp \beta \altsetst \norm{\beta}_1 \leq t\}$, and $L(y, y') = (y-y')^2$. 

Suppose that the examples are generated as $X \sim \D_x$, $Y = X\tp\beta_0 + U$ for $U \sim \D_u$ independent of $X$ with $\E[U] = 0$. Then assuming that $\norm{\beta_0}_1$ is bounded above by some known $t'$, as well as bounds on the fourth moments of $X$ an $U$, our inferential framework allows us to construct a valid confidence interval for $\beta_0$.

To this end, we first note that the risk of a parameter $\beta$ is given by 
\begin{equation*}
    R(\beta) 
    = \E[(X\tp\beta_0 + U - X\tp\beta)^2]
    = \E[(X\tp(\beta_0 - \beta) + U)^2].
\end{equation*}
We can then find a uniform convergence bound:
\begin{equation*}
    \Pr_{S\sim\D^m}\qty[\sup_\beta |R(\beta) - \rhat_S(\beta)| \leq \varepsilon] = \Pr_{S\sim\D^m}\qty[\sup_\beta\qty|R(\beta) - \frac{1}{m}\sum_{i=1}^m (X_i\tp(\beta_0 - \beta) + U_i)^2| \leq \varepsilon]
\end{equation*}
Since our parameter space is compact and the function we are taking the supremum over is continuous, we can substitute $\beta_0 - \beta = \widetilde{\beta}$ for some $\widetilde{\beta}$ a function of $\beta_0$, $(X_1, \ldots, X_m)$, and $(U_1, \ldots, U_m)$ that attains that supremum:
\begin{align*}
    \Pr_{S\sim\D^m}\qty[\sup_\beta |R(\beta) - \rhat_S(\beta)| \leq \varepsilon] 
    &= \Pr_{S\sim\D^m}\qty[\qty|\E[(X\tp\widetilde{\beta} + U)^2] - \frac{1}{m}\sum_{i=1}^m (X_i\tp\widetilde{\beta} + U_i)^2| \leq \varepsilon] \\
    &\geq 1 - \frac{1}{\varepsilon^2}\var\qty[\frac{1}{m}\sum_{i=1}^m (X_i\tp \widetilde{\beta} + U_i)^2] \\
    &= 1 - \frac{1}{m\varepsilon^2}\var[(X\tp\widetilde{\beta} + U)^2]
\end{align*}
For uniform convergence, we require this to be at least $1-\alpha$ for all $\widetilde{\beta}$, which has $1$-norm at most $t+t'$. Thus, we only need to have bounds on the fourth moments of $X$ and $U$ in order to obtain valid confidence sets for $\beta_0$. Notice that if $\norm{\beta_0}_1 \leq t$ (so that the hypothesis class contains $\beta_0$), then bounding $\var[(X\tp\widetilde{\beta} + U)^2]$ is exactly the same as bounding $\var[Y^2]$ when the $\beta_0$ generating $Y$ is allowed to be ``twice as big" as our model thinks it should be.

As in the previous example, one can use PAC-Learning theory to obtain uniform convergence functions that that do not rely on information about the moments of the data-generating distribution, but yield weaker bounds. Inversely, one can often obtain stronger bounds with more assumptions on the model. For example, a common assumption in regression is that the observed data is $\sigma^2$-subgaussian. Recall that a random variable $Z$ is $\sigma^2$-subgaussian if
\begin{equation*}
    \E\qty[\exp(t(Z - \E[Z]))] \leq \exp(\frac{t^2\sigma^2}{2})
\end{equation*}
for all $t\in\R$. The square of a $\sigma^2$-subgaussian random variable is $(32\sigma^4, 4\sigma^2)$-subexponential \citep[see Supplementary Material B]{honorio2014}, in the sense that
\begin{equation*}
    \E\qty[\exp(t(Z^2 - \E[Z^2]))] \leq \exp(\frac{t^2 \cdot 32\sigma^4}{2})
\end{equation*}
for all $|t| \leq 1/(4\sigma^2)$. Using standard concentration inequalities for subexponential random variables, we then have that for i.i.d. $\sigma^2$-subgaussian random variables $Z_1, \ldots, Z_m$,
\begin{equation*}
    \Pr[\qty|\E[Z^2] - \frac{1}{m}\sum_{i=1}^m Z_i^2| \leq \varepsilon] \geq 1 - 2\exp(-\min\qty(\frac{m\varepsilon^2}{64\sigma^4}, \frac{m\varepsilon}{8\sigma^2})).
\end{equation*}
Thus, by assuming $\sigma^2$-subgaussianity of $X_i\tp\widetilde{\beta} + U_i$, we can arrive at much sharper bounds for inference on $\beta_0$ that decay exponentially with the sample size. 
}\end{example}

\begin{example}\label[example]{ex:neuralnet}
\normalfont{
Consider a regularized neural network with $10$ nodes in one hidden layer using a sigmoid activation function. The hypothesis class for this model is
\begin{equation*}
    \H = \left\{\vb*{x}\mapsto \sigma\qty(\sum_{i=1}^{10} \vb*{w}^\top \sigma(\vb*{U}\vb*{x})) \altsetst \norm{\vb*{U}_i}_2 \leq M, \norm{\vb*{w}}_1 \leq \lambda \right\}.
\end{equation*}
Unlike in the previous examples, we typically have little to no information about the data-generating examples, and so typically fall back to the distribution-free bounds from classical PAC-Learning theory. A uniform convergence function for this model with the $L^1$ loss function can be computed via the Rademacher complexity \citep[see Theorem 3.3 and Exercise 3.11]{mohri2018}:
\begin{equation*}
    \Pr[\sup_{\vb*{U}, \vb*{w}} \qty|R(\vb*{U}, \vb*{w}) - \rhat_S(\vb*{U}, \vb*{w})| \leq 2\widehat{\mathfrak{R}}_S(\H) + 3\sqrt{\frac{\log(4/\alpha)}{2m}}] \geq 1-\alpha
\end{equation*}
where $\widehat{\mathfrak{R}}_S$ is the empirical Rademacher complexity, which is upper bounded by
\begin{equation*}
    \widehat{\mathfrak{R}}_S(\H) \leq \frac{2M\cdot\lambda}{m}\sqrt{\sum_{i=1}^m\norm{x_i}_2^2}.
\end{equation*}
Although the exact values of the parameters $\vb*{U}$ and $\vb*{w}$ are often not of direct interest (as interpretability is not a large concern for modern, more complex machine learning models), what is often of interest is whether or not the parameters lie in a particular low-dimensional subspace or attain a certain level of sparsity---assertions that may be tested by the standard method of inverting the confidence set $\Thetahat_S^\varepsilon$. We also once again emphasize that although risk minimization is unlikely to have a unique solution in the case of neural networks, our confidence sets cover \textit{all} risk minimizing $(\vb*{U}, \vb*{w})$ pairs with the nominal level of coverage, as these pairs constitute $\Theta_0^0$.
}
\end{example}

\begin{example}
\normalfont{
In the previous examples, the machine learning model $(\H, L)$ had the uniform convergence property, allowing for inference on the risk minimizer. To illustrate the necessity of the uniform convergence property, we now examine the behavior of a model that does not have a uniform convergence function. 

Suppose $(X_1, Y_1)\ldots, (X_m, Y_m)$ is an i.i.d. sample such that $X_1, \ldots, X_m \iid \operatorname{Uniform}(0, 1)$ and each $Y_i = 3X_i^2 + \operatorname{Uniform}(-0.1, 0.1)$. Let us perform polynomial regression using the $L^2$ loss on the data; to ensure that the uniform convergence property does not hold, we do not impose any restriction on the maximum degree, so that
\begin{equation*}
    \H = \left\{x\mapsto \sum_{n=0}^\infty \theta_n x^n \mid \theta \in c_{00}\right\}
\end{equation*}
where $c_{00}$ is the space of eventually vanishing sequences in $\R$. It is straightforward to show that the risk minimizer is given by $\theta_0 = (0, 0, 3)$, inducing the risk minimizing hypothesis $h(x\altgiven \theta_0) = 0 + 0x + 3x^2$ with $R(\theta_0) \approx 0.00333$. 

\begin{figure}[htbp]
    \centering
        \includegraphics[scale=.37,trim=34 0 45 20,clip=true]{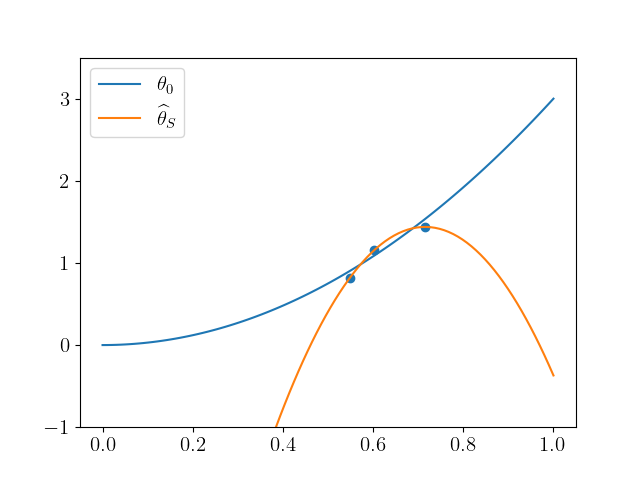}
        \includegraphics[scale=.37,trim=34 0 45 20,clip=true]{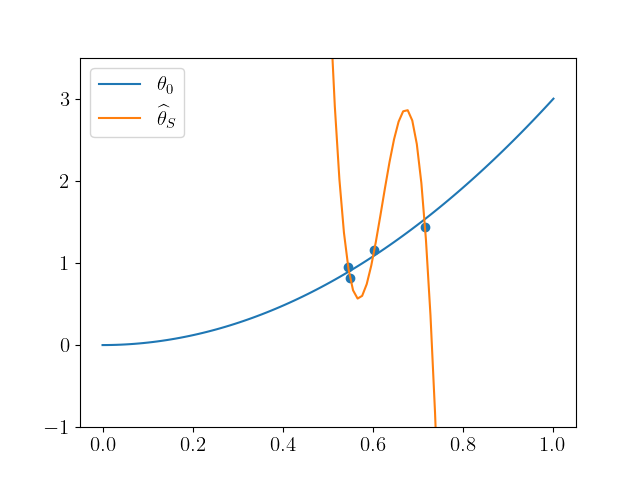}
        \includegraphics[scale=.37,trim=34 0 45 20,clip=true]{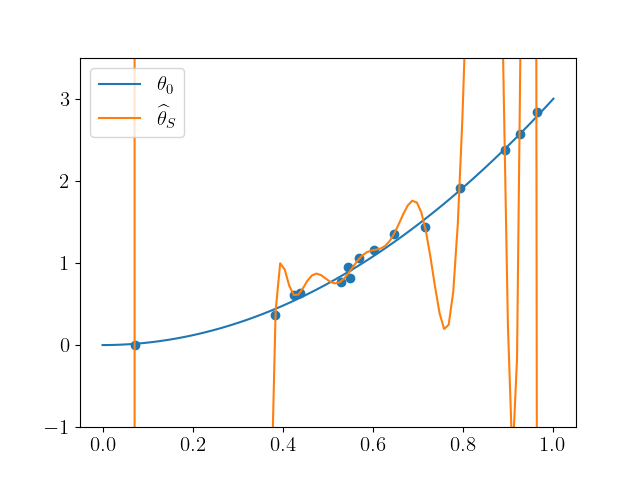}
    % \begin{subfigure}[t]{\textwidth}
    %     \centering
    %     \includegraphics[width=0.32\textwidth]{./images/m3.png}
    % \end{subfigure}
    % \begin{subfigure}[t]{\textwidth}
    %     \centering
    %     \includegraphics[width=0.32\textwidth]{./images/m4.png}
    % \end{subfigure}
    % \begin{subfigure}[t]{\textwidth}
    %     \centering
    %     \includegraphics[width=0.32\textwidth]{./images/m15.png}
    % \end{subfigure}
\caption{The polynomials induced by the risk minimizer $\theta_0$ and the ERM $\thetahat_S$ at sample sizes $m=3$, $m=4$, and $m=15$.}
\label{fig:polynomial}
\end{figure}

\Cref{fig:polynomial} demonstrates that when using empirical risk minimization in this problem, we have that $\rhat_S(\thetahat_S) = 0$ at all sample sizes, but the risk of the ERM actually diverges, as does the ERM itself. Indeed, at sample size $m=3$, the ERM is $\thetahat_S = (-9.97, 31.90, -22.30)$ with $R(\thetahat_S) \approx 13.34$; at sample size $m=4$, we have $\thetahat_S = (935, -4581, 7438, -3996)$ with $R(\thetahat_S) \approx 77962$; at sample size $m=15$, we have each entry of $\thetahat_S$ to be between $10^6$ and $10^{12}$ in magnitude, with $R(\thetahat_S) \approx 2 \times 10^{10}$. 

Because $(\H, L)$ does not have the uniform convergence property, $\rhat_S(\thetahat_S)$ and $R(\theta_0)$ have essentially no relationship to each other, so the distance (in any reasonable metric) between $\thetahat_S$ and $\theta_0$ increases without bound as the sample size increases. As a result, no $\varepsilon$-neighborhood of $\thetahat_S$ can cover $\theta_0$ with probability $1-\alpha$ as the sample size grows---with probability 1, the distance between $\thetahat_S$ and $\theta_0$ will eventually be larger than the chosen $\varepsilon$. This clearly illustrates that without the uniform convergence property, one cannot perform inference on $\theta_0$ by using a confidence set centered at the ERM.
}
\end{example}

\section{Efficient Assignment of Confidence}\label{sec:distbn}
As alluded to in \Cref{ex:neuralnet},  it is often the case that there is fixed subset of the parameter space (e.g. a lower-dimensional manifold) that the practitioner is interested in performing uncertainty quantification about. For example, the practitioner may want to quantify how confident they are that for $\theta_0 = (\theta^{(1)}, \ldots, \theta^{(n)})$, it is the case that $\theta^{(1)} = \theta^{(2)} = 0$.

Bayesian inference easily allows one to calculate the posterior probability of any such arbitrary subset of the parameter space, but frequentist procedures have more difficulty in this regard. The primary problem is illustrated in \Cref{fig:confefficiency}; supposing that confidence for arbitrary regions of the parameter space is done by the usual method of inverting the confidence set
\begin{equation*}
    \operatorname{Conf}(A) = \sup\left\{1-\alpha \mid \Thetahat_S^{\varepsilon(m, \alpha)}\supseteq A\right\},
\end{equation*}
where our radius $\varepsilon$ now depends on the sample size and significance level, we arrive at assignments of confidence that are clearly inefficient. Consequently, hypothesis testing for $H_0: \theta_0 \in A$ in this manner is not particularly powerful. 

In this section, we show that the properties of the confidence sets our ML framework for inference do allow for a more efficient assignment of confidence to sufficiently regular subsets of the parameter space.

\begin{figure}
\centering
\begin{tikzpicture}
    \def\radius{50}
    \def\radiuspt{\radius pt}
    % \Thetahat^\varepsilon_S
    \node[circle, draw, minimum size={2*\radiuspt}, label=above right:$\Thetahat_S^{\varepsilon(m, \alpha)}$] (ci) at (0,0){};
    %\thetahat_S
    \node[circle, draw, fill=black, inner sep=0pt, minimum size = 3pt, label=right:$\thetahat_S$](thetahat) at (0, 0){};
    % A
    \node[circle, draw, dashed, minimum size={1.5*\radiuspt}, label= right:$A$](A)at (-0.25*\radiuspt, 0){};
    % B
    \node[circle, draw, dashed, minimum size={0.1*\radiuspt}, label= right:$B$](B)at (-0.9*\radiuspt, 0){};
\end{tikzpicture}
\caption{When inverting $\Thetahat_S^{\varepsilon(m, \alpha)}$ to determine our confidence in arbitrary regions of the parameter space, both regions $A$ and $B$ are assigned the same confidence $1-\alpha$. However, since $B \subseteq A$, it is clear that we should have $\operatorname{Conf}(B) \leq \operatorname{Conf}(A)$.}
\label{fig:confefficiency}
\end{figure}
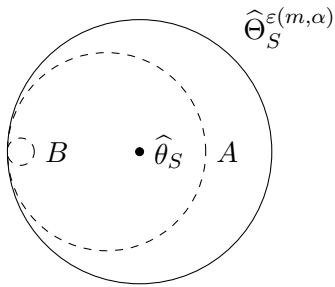

\subsection{Random Sets and Imprecise Probability}
In order to understand how our framework allows for more efficient assignments of confidence, it is critical to understand the properties of our set $\Thetahat_S^\varepsilon$ of $\varepsilon$-AERMs. It is of note that $\Thetahat_S^\varepsilon$ is random due to the dependence on the random sample $S$, but differs from usual random variables or random vectors by being set-valued. Thus, to fully understand the properties of $\Thetahat_S^\varepsilon$, we first need to understand random sets.

We begin with the formal definition of a random set \citep{molchanov2017}:
\begin{definition}
Let $(\Theta, \tau)$ be a Polish space and $(\Omega, \Sigma)$ a measurable space. A function ${X:\Omega \rightarrow 2^\Theta}$ is a closed random set if $X(\omega)$ is closed for every $\omega \in \Omega$ and 
\begin{equation*}
    \{\omega \setst X(\omega) \cap U \neq \varnothing\}\in \Sigma
\end{equation*}
for every $U \in \tau$.
\end{definition}
This definition can be generalized to non-Polish spaces as well as to non-closed random sets. 

Whereas the distribution of a random variable is given by the cumulative distribution function, the distribution of a closed random set $X$ is given by the \textit{belief} and \textit{plausibility} functions
\begin{align*}
    \bel(A) &= \Pr[X \subseteq A] \\
    \pl(A) &= \Pr[X \cap A \neq \varnothing]
\end{align*}
where $A$ is a Borel set \citep{molchanov2017}. It is straightforward to check that these two functions are duals of each other, in the sense that
$\bel(A) = 1 - \pl(\Theta\setminus A)$ and $\pl(A) = 1 - \bel(\Theta\setminus A)$.
As such, knowing either one of the belief or plausibility functions immediately yields the other. A very useful property of belief and plausibility functions is given by Choquet's Theorem \citep{matheron1974}:
\begin{theorem}
Let $\Theta$ be a locally compact Hausdorff second countable topological space. Then the plausibility function induced by a random closed set in $\Theta$ is $\infty$-alternating and upper semicontinuous. That is,
\begin{equation*}
    \pl\qty(\bigcup_{i=1}^k A_i) \leq \sum_{\varnothing\neq I\subseteq\{1, \ldots, k} (-1)^{|I| + 1}\pl\qty(\bigcap_{i\in I} A_i)
\end{equation*}
for every positive integer $k$, and $\pl(K_n) \rightarrow \pl(K)$ for every sequence of compact sets $(K_n)$ converging from above to a compact set $K$.
\end{theorem}
Similarly, we would have that belief functions are $\infty$-monotone (reversing the inequality) and lower semicontinuous. Note that probability measures are both $\infty$-alternating and $\infty$-monotone; this fact is often referred to as the principle of inclusion-exclusion.

The belief and plausibility functions are often additionally conditioned on the event $X\neq \varnothing$ so that we have the convenient properties that $\bel(\varnothing) = \pl(\varnothing) = 0$ and $\bel(\Theta) = \pl(\Theta) = 1$; note that these properties trivially hold if the random set is almost surely nonempty. If, in addition, the $\infty$-monotone and $\infty$-alternating conclusions from Choquet's Theorem also hold, then the belief and plausibility functions fall under the purview of the Dempster-Shafer theory of evidence \citep{dempster1967,shafer1976}. Within this framework, \citet{dempster1967} suggests that given a random set $X$ that corresponds to a quantity of interest $\theta_0\in\Theta$, we can interpret $\bel(A)$ and $\pl(A)$ as \textit{lower} and \textit{upper} probabilities for the assertion $\theta_0 \in A$. That is to say that $\bel(A)$ can be interpreted as the probability that $\theta_0\in A$ is ``known to be true," $1-\pl(A)$ can be interpreted as the probability that $\theta_0\in A$ is ``known to be false," and the remaining probability $\pl(A) - \bel(A)$ can be interpreted as a residual ``don't know" probability \citep{dempster2008}. As a simple example, consider the trivial random set $X = \Theta$; then $\bel(A) = 0$ and $\pl(A) = 1$ for every nonempty proper subset $A$ of $\Theta$---Dempster's interpretation would be that we have no evidence for the truth or falsity of $\theta_0\in A$ for any such $A$, and so have absolutely no knowledge of the location of $\theta_0$. 

A formal theory of \textit{imprecise probability} can be worked out from belief and plausibility functions. For example, \citet{shafer1979} demonstrates how to extend belief and plausibility functions from union- or intersection-closed subsets of $\Theta$ (e.g. the Borel sets) to the entire set of subsets of $\Theta$ via the notions of allocations and allowments of probability; this is analogous to how the Lebesgue measure is extended from open sets to the Lebesgue measurable sets. \citet{shafer2016} briefly discusses the theory of Choquet integration for beliefs and plausibilities (analogous to Lebesgue integration for probabilities) as well as the notions of independence, joint imprecise probability distributions, and conditional imprecise probability. In particular, \citet{shafer2016} proves that Bayes's rule holds for plausibility functions even when $\Theta$ is infinite; this theorem is known as Dempster's rule of conditioning.

The Dempster-Shafer framework is not the only approach to imprecise probability. An alternative to working with belief and plausibility functions is to work with \textit{necessity} and \textit{possibility} functions. Possibility functions still keep the properties $\operatorname{pos}(\varnothing) = 0$ and $\operatorname{pos}(\Theta) = 1$, but impose a maxitivity assumption:
\begin{equation*}
    \operatorname{pos}\qty(\bigsqcup_{i\in I} A_i) = \sup_{i\in I}\operatorname{pos}(A_i)
\end{equation*}
for every index set $I$ \citep[Definition II.5.$\eta$]{dubois1980}. The necessity function is once again the dual of the possibility function. The necessity and possibility function approach to imprecise probability has a variety of nice properties due to the maxitivity assumption, but is limited in its ability to describe the behavior of random sets: It can be shown that the plausibility function induced by a random set $X$ is a possibility function if and only if the realizations of $X$ are nested (i.e. $X(\omega_1) \subseteq X(\omega_2)$ or $X(\omega_2) \subseteq X(\omega_1)$ for all $\omega_1, \omega_2 \in \Omega$) \citep[Theorem 10.1]{shafer1976}. 

Another approach to developing imprecise probability theory is via previsions; this theory is discussed in detail in \citet{walley1991}. Given a linear space $\mathcal{F}$ of functions $f:\Theta \rightarrow \R$ (called ``gambles"), the lower and upper previsions defined on $\mathcal{F}$ are
\begin{align*}
    \underline{E}[f] &= \sup\{\mu \in \R \setst f - \mu \text{ is desirable}\} \\
    \overline{E}[f] &= \inf\{\mu\in\R \setst \mu-f \text{ is desirable}\},
\end{align*}
where we call a gamble ``desirable" if one would accept the gamble if offered. Note that previsions can then be naturally extended to larger classes of gambles (e.g. to include all indicator functions). These previsions act analogously to expectations in probability theory, and so upper and lower probabilities can be defined as the upper and lower previsions of indicator functions. 

\subsection{Validity of ML Models}
Since the set $\Thetahat_S^\varepsilon$ of $\varepsilon$-AERMs is a random set, its distribution is determined by the induced belief and plausibility functions:
\begin{align*}
	\bel_\varepsilon(A) &= \Pr_{S\sim\D^m}[\Thetahat_S^\varepsilon \subseteq A \given \Thetahat_S^\varepsilon \neq \varnothing] \\
	\pl_\varepsilon(A) &= \Pr_{S\sim\D^m}[\Thetahat_S^\varepsilon \cap A \neq \varnothing \given \Thetahat_S^\varepsilon \neq \varnothing].
\end{align*}
Notice that these functions are not well-defined when $\Thetahat_S^\varepsilon$ is almost surely empty; this is only possible when $\varepsilon = 0$ and the ERM almost surely does not exist. 
\begin{example}\normalfont{
Consider $\X = \{\varnothing\}$, $\Y = \R$, and $\H = \{x \mapsto \theta \mid \theta\in\mathbb{Q}\}$ with $L(y, y') = |y - y'|$, where the data-generating distribution is a point mass at an irrational number (e.g. $\Pr(Y = \pi) = 1$). In this example, $\thetahat_S$ never exists, since there always exists a closer rational approximation to an irrational number. Thus, $\Thetahat_S^0$ is always empty, and the belief and plausibility for $\varepsilon=0$ are not well-defined. 
}\end{example}

Because this situation can be circumvented in practice by instead considering $\varepsilon$-plausibilities for any choice of $\varepsilon > 0$, we assume for the remainder of this paper that the $\varepsilon$-plausibility is well defined.

Knowledge of the distribution of $\varepsilon$-AERMs can be used to assign a confidence to the proposition that a given region of the parameter space contains the risk minimizer. In particular, we will show in this section that sets of low plausibility cannot contain the risk minimizer. To this end, we first define the notion of validity for an ML model:
\begin{definition}\label[definition]{def:validmodel}
At a fixed sample size, the model $(\H, L)$ is valid at level $\alpha$ and tolerance $\varepsilon$ if for every Borel set $A\subseteq \Theta$ such that there exists $\delta \geq 0$ so that $\Theta_0^\delta \subseteq A$ and $\Theta_0^\delta \neq \varnothing$, we have that $\pl_\varepsilon(A) \geq 1-\alpha$.
\end{definition}
In other words, an ML model is valid if every nonempty $\delta$-neighborhood of the risk minimizer has a high plausibility. 
Once again, we note that when the risk minimizer $\theta_0$ exists, it is sufficient to only consider $\delta = 0$, as any $A$ containing $\Theta_0^\delta$ for $\delta > 0$ must contain $\theta_0\in\Theta_0^0$ itself.

\begin{lemma}\label[lemma]{lem:ucfnoninceasing}
Let $(\H, L)$ have the uniform convergence function $f$. Then $f$ is non-increasing in its first argument.
\end{lemma}
\begin{proof}
Fix $\varepsilon_1,\alpha > 0$. Then
\begin{equation*}
    \Pr_{S\sim\D^m}\qty[\sup_{\theta\in\Theta} |R(\theta) - \rhat_S(\theta)| \leq \varepsilon_1] \geq 1-\alpha \enspace\text{ if } m\geq f(\varepsilon_1, \alpha).
\end{equation*}
Now let $\varepsilon_2 > \varepsilon_1$. Then we have that
\begin{equation*}
    \Pr_{S\sim\D^m}\qty[\sup_{\theta\in\Theta} |R(\theta) - \rhat_S(\theta)| \leq \varepsilon_1] \leq \Pr_{S\sim\D^m}\qty[\sup_{\theta\in\Theta} |R(\theta) - \rhat_S(\theta)| \leq \varepsilon_2]
\end{equation*}
and so
\begin{equation*}
    \Pr_{S\sim\D^m}\qty[\sup_{\theta\in\Theta} |R(\theta) - \rhat_S(\theta)| \leq \varepsilon_2] \geq 1-\alpha \enspace\text{ if } m\geq f(\varepsilon_1, \alpha).
\end{equation*}
We hence have that
\begin{equation*}
    w(\varepsilon, \alpha) = \begin{cases}
        f(\varepsilon, \alpha) & \text{if } \varepsilon \neq \varepsilon_2 \\
        f(\varepsilon_1, \alpha) & \text{if } \varepsilon = \varepsilon_2
    \end{cases}
\end{equation*}
witnesses the uniform convergence property of $(\H, L)$.
Therefore, by the definition of the uniform convergence function (\Cref{def:ucf}),
\begin{equation*}
    f(\varepsilon_2, \alpha) \leq \lceil w(\varepsilon_2, \alpha)\rceil = \lceil f(\varepsilon_1, \alpha)\rceil = f(\varepsilon_1, \alpha)
\end{equation*}
since $f(\varepsilon_1, \alpha)$ is an integer, as desired.
\end{proof}

\begin{corollary}\label[corollary]{cor:valid}
Let $(\H, L)$ have uniform convergence function $f$. If the risk minimizer $\theta_0$ exists, then $(\H, L)$ is valid at level $\alpha$ and tolerance $\varepsilon$ if $m \geq f(\varepsilon/2, \alpha)$. Otherwise, $(\H, L)$ is valid at level $\alpha$ and tolerance $\varepsilon$ if $m \geq \inf_{\delta > 0} f((\varepsilon - \delta)/2, \alpha)$.
\end{corollary}
\begin{proof}
Suppose that the risk minimizer $\theta_0$ exists. Then if $A$ is a Borel set and $\theta_0 \in A$, we have that
\begin{align*}
    \pl_\varepsilon(A) 
    &= \Pr[\Thetahat_S^\varepsilon \cap A \neq \varnothing \given \Thetahat_S^\varepsilon\neq \varnothing] \\
    &\geq \Pr[\Thetahat_S^\varepsilon \cap \{\theta_0\}  \neq \varnothing \given \Thetahat_S^\varepsilon\neq \varnothing] \\
    &= \Pr[\Thetahat_S^\varepsilon \ni \theta_0 \text{ and } \Thetahat_S^\varepsilon\neq\varnothing]/\Pr[\Thetahat_S^\varepsilon \neq\varnothing] \\ 
    &\geq \Pr[\Thetahat_S^\varepsilon \ni \theta_0]/1,
\end{align*}
which has probability at least $1-\alpha$ by the result of \Cref{thm:eseps_ci}.

Now suppose that the risk minimizer $\theta_0$ does not exist. Then let $\zeta > 0$ be such that $\Theta_0^\zeta \subseteq A$. Without loss of generality, we also let $\zeta < \varepsilon$. Then for any $\delta \in (0, \zeta)$, we have that
\begin{equation*}
    \pl_\varepsilon(A) = \Pr[\Thetahat_S^\varepsilon \cap A \neq \varnothing] \geq \Pr[\Theta_0^\delta \subseteq \Thetahat_S^\varepsilon]
\end{equation*}
which has probability at least $1-\alpha$ by the result of \Cref{thm:eseps_bigcs} if $m \geq f((\varepsilon-\delta)/2, \alpha)$. Since this is true for every $\delta \in (0, \zeta)$ and $f$ is non-increasing in its first argument by \Cref{lem:ucfnoninceasing}, we have validity if $m \geq \inf_{\delta \in(0,\zeta)} f((\varepsilon - \delta)/2, \alpha) = \inf_{\delta > 0} f((\varepsilon - \delta)/2, \alpha)$, as desired.
\end{proof}
Note that when $f(\cdot, \alpha)$ is left-continuous at $\varepsilon$, $\inf_{\delta > 0} f((\varepsilon - \delta)/2, \alpha) = f(\varepsilon/2, \alpha)$, and so the sample complexity is the same regardless of the existence of the risk minimizer. However, this is not the case in general, so the sample size necessary for validity is generally easier to attain when the risk minimizer does exist.

The contrapositive of the above corollary is that if a model is valid at level $\alpha$ and tolerance $\varepsilon$, then $\pl_\varepsilon(A) < 1 - \alpha$ implies that $\theta_0\not\in A$. Thus, knowledge of the distribution of the set of $\varepsilon$-AERMs yields important knowledge about the location of the risk minimizer. We restate this contrapositive as a theorem in its own right:
\begin{theorem}\label{thm:contrapos}
Let $(\H, L)$ be valid at level $\alpha$ and tolerance $\varepsilon$. If $\pl_\varepsilon(A) < 1-\alpha$  then $\theta_0 \not \in A$.
\end{theorem}

~

This idea of determining optimal values for parameters via \Cref{thm:contrapos} can be useful for hypothesis classes that make use of a tuning parameter $\gamma$, as calculating the belief and plausibility of sets of the form $A = \{\gamma \in [a, b]\}$ may help reduce the search space for the tuning parameter.

\begin{quote}
\begin{remark}
\normalfont{
    It may initially seem strange to discuss performing ``inference" on a tuning parameter. However, it is undoubtedly the case that there exists some optimal value for the tuning parameter, as it is typically estimated via cross-validation. As such, we can still perform uncertainty quantification for this optimal tuning parameter in the same way we do so for more traditional feature parameters.
}
\end{remark}
\end{quote}
\begin{example}\normalfont{
In continuation of \Cref{ex:lasso}, consider LASSO estimation with data generated by $Y = X\beta_0 + U$. Our hypothesis class is $\H = \bigcup_{t' \leq t} \H_{t'}$, where $\H_{t'} = \{x\mapsto x\tp\beta \altsetst \norm{\beta}_1 \leq t'\}$ and $t$ is an upper bound on the LASSO tuning parameter provided by the practitioner. The value of the tuning parameter $t'$ is typically chosen via cross-validation over the interval $[0, t]$, since the optimal value for $t'$ is $t_0 = \min(t, \norm{\beta_0}_1)$ but $\norm{\beta_0}$ is typically unknown.

We draw examples $X\sim\operatorname{Unif}(-1, 1)^p$ and $U \sim \operatorname{Unif}(-1, 1)$. A uniform convergence function is then given by
\begin{equation*}
    f(\varepsilon, \alpha) = 1 - 2\exp(-\frac{m\varepsilon}{8(\norm{\beta_0} + 1)^2}).
\end{equation*}
This then allows us to compute an $\varepsilon$ necessary for validity at a given sample size $m$ and significance level $\alpha$, given (an upper bound on) the magnitude of $\norm{\beta_0}_1$. 

We randomly select a $\beta$ from $\operatorname{Unif}(-1, 1)^p$ and generate $m=1000$ training examples, computing $\pl_\varepsilon(A)$ for $A = \{\beta \altsetst \norm{\beta}_1 \leq t'\}$ for various choices of tuning parameter $t'$. That is, we conduct hypothesis tests for $H_0: t_0\leq t'$ for $t'$ ranging over an interval. In \Cref{fig:lasso_pl}, we show the results in an example where $p=10$, $t = 10$, and $\norm{\beta_0}_1 \approx 3.34$ (so that the optimal tuning parameter is $t_0 \approx 3.34$). 

We see that for tuning parameters $t'$ less than about $1.3$, the plausibility is less than $0.95$ and so cannot possibly be optimal. With access to this information, a practitioner now knows that it is only worthwhile to cross-validate over the range $[1.3, 10]$ rather than the entire range $[0, 10]$ since all tuning parameters less than $1.3$ are provably suboptimal.

\begin{figure}[H]
\centering
\includegraphics[width=0.60\textwidth]{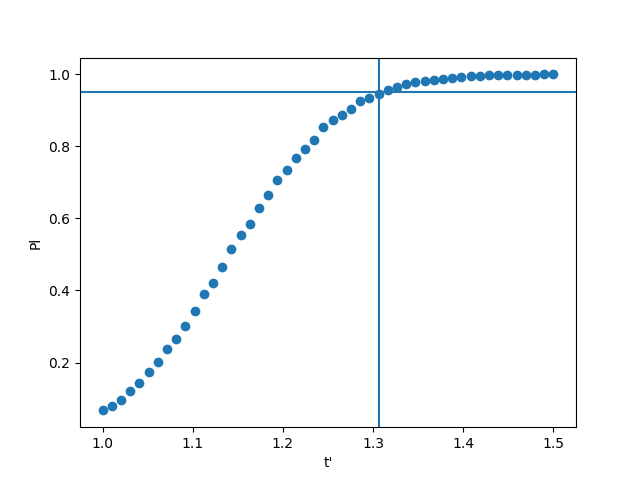}
\caption{Estimates for the plausibility of the set $\{\beta \altsetst \norm{\beta}_1 \leq t'$\} at significance level $0.05$ for different tuning parameters $t'$. A vertical line is plotted at the $t'$ that maintains plausibility at least 0.95. Plausibilities were estimated via Monte-Carlo simulation with $10000$ replicates.}\label{fig:lasso_pl}
\end{figure}
}
\end{example}

\subsection{Bootstrapping Belief and Plausibility}\label{sec:bootstrap}
Actually calculating the plausibility given a single sample is not particularly straightforward, since typically the sampling distribution $\D$ is completely unknown and the ``sample plausibility" $\widehat{\pl}_\varepsilon(A) = I(\Thetahat_S^\varepsilon \cap A \neq \varnothing)$ is always either zero or one.
Hence, it makes sense to estimate belief and plausibility via bootstrapping:
\begin{equation*}
	\widehat{\pl}^{\text{boot}, S}_{\varepsilon}(A) = \frac{1}{B}\sum_{i=1}^B I(\Thetahat_{S_i}^{\varepsilon} \cap A \neq \varnothing)
\end{equation*}
where $S_1, \ldots, S_B$ are subsamples drawn uniformly and independently from the observed sample $S$ such that $\Thetahat_{S_i}^\varepsilon$ is nonempty. Let us define
\begin{equation*}
	\pl^{\text{boot},S}_\varepsilon(A) = \lim_{B\rightarrow\infty} \frac{1}{B}\sum_{i=1}^B I(\Thetahat_{S_i}^{\varepsilon} \cap A \neq \varnothing)
\end{equation*}
and similar for the belief. Notice that the quantity $\pl_\varepsilon^{\text{boot}, S}$ is the expected value of $\widehat{\pl}_\varepsilon^{\text{boot}, S}$ over the resampling distribution. Thus, it is convenient for theoretical purposes to consider $\pl_\varepsilon^{\text{boot}, S}$ as ``the" bootstrapped plausibility. However, it is $\widehat{\pl}_\varepsilon^{\text{boot}, S}$ that is actually used by the practitioner. Luckily, these two quantities are quite close to each other even for modestly large $B$ (as we will see in \Cref{thm:samplebootplvalid}), so the practitioner (approximately) enjoys the guarantees of the bootstrapped plausibility.

Once again, these bootstrapped quantities are not well-defined when $\Thetahat_{S_i}^\varepsilon$ is empty with probability $1$. This happens if and only if $\varepsilon = 0$ and the ERM on each subsample almost surely does not exist. Because a practitioner can always avoid such an event by simply choosing any $\varepsilon > 0$ when they notice that the ERMs never exist, we will assume throughout this section that the model and sampling distribution are such that these bootstrapping quantities are well-defined. \\

In order for bootstrapping to be useful, we require slightly stronger conditions on our ML model:

\begin{definition}\label[definition]{def:strongucp}
An ML model $(\H, L)$ has the \textit{strong} uniform convergence property with respect to the data-generating distribution $\D$ over $\X\times \Y$ if there exists a function ${f:\R^+\times\R^+\rightarrow \R}$ such that for any $\varepsilon > 0$ and $\alpha > 0$, if $m \geq f(\varepsilon, \alpha)$ then 
\begin{equation*}
    \Pr_{S\sim\D^m}\qty[\sup_{\theta\in\Theta}|R(\theta) - \rhat_S(\theta)| \leq \varepsilon] \geq 1-\alpha
\end{equation*}
(i.e. \Cref{def:ucp} holds) and
\begin{equation*}
    \inf_{\widetilde{\D} \in \mathcal{B}}\Pr_{S\sim\widetilde{\D}^m}\qty[\sup_{\theta\in\Theta}|R_{\widetilde{\D}}(\theta) - \rhat_S(\theta)| \leq \varepsilon] \geq 1-\alpha,
\end{equation*}
where $\mathcal{B}$ denotes the set of bootstrap resampling distributions over $\D$; that is, the set of $\operatorname{Unif}(S)^m$ for every possible $S\sim\D^m$. We call any such $f$ a \textit{witness} to the strong uniform convergence property.
\end{definition}

In other words, the strong uniform convergence property requires that the uniform convergence property (\Cref{def:ucp}) hold over both the data-generating distribution and the bootstrap resampling distribution. Once again, we note that this condition remains much weaker than that of being a strong uniform Glivenko-Cantelli class (\Cref{def:sugc}).

\begin{definition}\label[definition]{def:strongucf}
Let $(\H, L)$ have the strong uniform convergence property, and let $\mathcal{W}$ be the set of all corresponding witnesses.
The strong uniform convergence function of $(\H, L)$ is defined to be
\begin{equation*}
    f(\varepsilon, \alpha) = \inf_{w\in\mathcal{W}} \lceil w(\varepsilon, \alpha)\rceil,
\end{equation*}
where $\lceil\,\cdot\,\rceil$ is the ceiling function.
\end{definition}

Due to the reliance on a random sample, bootstrapped plausibilities do not yield exact knowledge about the location of the risk minimizer. However, we find that bootstrapped plausibilities do yield this knowledge with high probability:

\begin{theorem}\label{thm:plboot}
Suppose $A$ is a Borel set and $\Theta_0^\delta\subseteq A$ for some $\delta > 0$. Suppose that $(\H, L)$ has the strong uniform convergence function $f$. If either
\begin{itemize}
\item[(a)] the risk minimizer $\theta_0$ exists, the empirical risk minimizer $\thetahat_S$ exists for every sample $S$, and we define $\varepsilon(m, \alpha) = {\inf\{\varepsilon \mid m \geq f(\varepsilon/2, \alpha)\}}$
\item[(b)] we define $\varepsilon(m, \alpha) = \inf\{\varepsilon \mid m \geq \inf_{\delta > 0} f((\varepsilon-\delta)/2, \alpha)\}$
\end{itemize}
then bootstrapping the plausibility function is asymptotically valid, in the sense that
\begin{equation*}
	\liminf_{m\rightarrow\infty} \Pr_{S\sim\D^m}[\pl^{\operatorname{boot},S}_{\varepsilon(m, \alpha)}(A) \geq 1-\alpha] \geq 1-\alpha.
\end{equation*}
In particular, in (a) the inequality holds for all $m$ such that $\varepsilon(m, \alpha) \leq \delta$, and in (b) the inequality holds for all $m$ such that $\varepsilon(m, \alpha) \leq \delta/2$
\end{theorem}
\begin{proof}
We first prove (a). By the strong law of large numbers, we have that
\begin{equation*}
	\pl^{\text{boot}, S}_{\varepsilon(m, \alpha)}(A) \overset{a.s.}{=} \Pr_{S_i \sim \operatorname{Unif}(S)^m}[\Thetahat_{S_i}^{\varepsilon(m, \alpha)} \cap A \neq \varnothing \mid \Thetahat_{S_i}^{\varepsilon(m, \alpha)} \neq \varnothing].
\end{equation*}
The right hand side is simply $\pl_{\varepsilon(m, \alpha)}(A)$ with respect to the uniform distribution on $S$; we know that this is at least $1-\alpha$ so long as $\thetahat_S \in A$ by Corollary \ref{cor:valid}. 
We thus have that
\begin{equation}\label{eq:plgeqconf}
	\Pr_{S\sim\D^m}[\pl^{\text{boot}, S}_{\varepsilon(m, \alpha)}(A) \geq 1-\alpha] \geq \Pr[\thetahat_S \in A].
\end{equation}
By hypothesis, $\Theta_0^\delta \subseteq A$. Hence,
\begin{equation}\label{eq:confgeqdeltabound}
    \Pr[\thetahat_S \in A] 
    \geq \Pr[\thetahat_S \in \Theta_0^\delta]
    = \Pr[R(\thetahat_S) \leq R(\theta_0) + \delta].
\end{equation}
We know that for any $m$, we have with probability at least $1-\alpha$ that
\begin{flalign*}
    &&R(\thetahat_S) &\leq \rhat_S(\thetahat_S) + \frac{\varepsilon(m, \alpha)}{2} & \text{(uniform convergence)}\\ 
    &&&\leq \rhat_S(\theta_0) + \frac{\varepsilon(m, \alpha)}{2} &\text{(definition of $\thetahat_S$)}\\
    &&&\leq R(\theta_0) + \varepsilon(m, \alpha) & \text{(uniform convergence)}
\end{flalign*}
Furthermore, as $m\rightarrow\infty$, we have that $\varepsilon(m, \alpha) \rightarrow 0$. Thus, there exists $M\in\N$ such that $\varepsilon(m, \alpha) \leq \delta$ for any $m\geq M$. Hence, we have for all $m\geq M$ that
\begin{equation*}
    \Pr_{S\sim\D^m}[R(\thetahat_S) \leq R(\theta_0) + \delta)] \geq \Pr_{S\sim\D^m}[R(\thetahat_S) \leq R(\theta_0) + \varepsilon(m, \alpha)] \geq 1-\alpha.
\end{equation*}
Combining this with equations \eqref{eq:plgeqconf} and \eqref{eq:confgeqdeltabound}, we arrive at
\begin{equation*}
    \liminf_{m\rightarrow\infty}\Pr_{S\sim\D^m}[\pl^{\text{boot}, S}_{\varepsilon(m, \alpha)}(A) \geq 1-\alpha] \geq 1-\alpha
\end{equation*}
as desired. \\

We now prove (b). As in the previous part, we have that $\pl^{\text{boot},S}_{\varepsilon(m, \alpha)}(A) \geq 1-\alpha$ so long as $\Thetahat_S^{\delta/4} \subseteq A$. Hence,
\begin{equation*}
    \Pr_{S\sim\D^m}[\pl^{\text{boot},S}_{\varepsilon(m, \alpha)}(A) \geq 1-\alpha] \geq \Pr[\Thetahat_S^{\delta/4} \subseteq A] \geq \Pr[\Thetahat_S^{\delta/4} \subseteq \Theta_0^\delta].
\end{equation*}
Let $\vartheta_0 \in \Theta_0^{\delta/4}$. By uniform convergence, we have with probability at least $1-\alpha$ that for all $\varthetahat_S\in \Thetahat_S^{\delta/4}$,
\begin{flalign*}
    &&R(\varthetahat_S) 
    &\leq \rhat_S(\varthetahat_S) + \frac{\varepsilon(m, \alpha)}{2} & \text{(uniform convergence)}\\
    &&&\leq \inf_{\theta\in\Theta} \rhat_S(\theta) + \frac{\delta}{4} + \frac{\varepsilon(m, \alpha)}{2} & \text{(definition of $\varthetahat_S\in\Thetahat_S^{\delta/4}$)} \\
    &&&\leq \rhat_S(\vartheta_0) + \frac{\delta}{4} + \frac{\varepsilon(m, \alpha)}{2} & \text{(definition of infimum)}\\
    &&& \leq R(\vartheta_0) + \varepsilon(m, \alpha) + \frac{\delta}{4} & \text{(uniform convergence)}\\
    &&&\leq \inf_{\theta\in\Theta} R(\theta) + \varepsilon(m, \alpha) + \frac{\delta}{2} & \text{(definition of $\vartheta_0\in\Theta_0^{\delta/4}$)}
\end{flalign*}
As $m\rightarrow\infty$, $\varepsilon(m, \alpha) \rightarrow 0$, so for large enough $m$, we have that $\varepsilon(m, \alpha) \leq \delta/2$, so
\begin{equation*}
    \Pr_{S\sim\D^m}\qty[\Thetahat_S^{\delta/4} \subseteq \Theta_0^{\delta}] = \Pr_{S\sim\D^m}\bigg[\bigcap_{\varthetahat_S\in\Thetahat_S^{\delta/4}}\left\{R(\varthetahat_S) \leq \inf_{\theta\in\Theta}R(\theta) + \delta\right\}\bigg] \geq 1-\alpha
\end{equation*}
for large $m$ as desired.
\end{proof}

It is worth noting that the sample size necessary for the validity of bootstrapping depends on the ``size" of $A$---i.e. the magnitude of the largest $\delta$ such that $\Theta_0^\delta \subseteq A$. If $\delta$ (and thus $A$) is large, only a small sample size $m$ is necessary for $\varepsilon(m, \alpha) \leq \delta/2$. Inversely, a small $\delta$ (and thus small $A$) requires a larger sample size for validity to hold. One might thus think that given a very large sample size, it would be reasonable to bootstrap on finite sets $A$; however, if $A$ is too small then no $\Theta_0^\delta$ will be a subset of $A$ for any $\delta > 0$, possibly hampering validity at every sample size. \\

The primary consequence of \Cref{thm:plboot} is that $1 - \pl_{\varepsilon(m, \cdot)}^{\text{boot}, S}(A)$ acts as an asymptotic $p$-value for the hypothesis $H_0: \theta_0 \in A$. Hence, we can assign an (asymptotic) confidence to the proposition that the risk minimizer lies in $A$:
\begin{equation*}
    \operatorname{Conf}(A) \approx \sup\{1-\alpha \mid \pl_{\varepsilon(m, \alpha)}^{\text{boot}, S}(A) \geq 1-\alpha\}.
\end{equation*}
Returning to figure \Cref{fig:confefficiency}, we now see that this assignment of confidence is more efficient than that from inversion of the confidence set, as if $B \subseteq A$, it is necessarily the case that $\pl^{\text{boot}, S}_\varepsilon(B) \leq \pl^{\text{boot}, S}_\varepsilon(A)$. Consequently, hypothesis testing for $H_0: \theta_0\in A$ via bootstrapping is in general more powerful than the standard method of hypothesis testing via the inversion of the confidence set.

\begin{theorem}\label{thm:samplebootplvalid}
Suppose that either set of  hypotheses in \Cref{thm:plboot} holds. Then for any $\gamma\in(0, \alpha)$,
\begin{equation*}
    \liminf_{m\rightarrow\infty}\Pr_{S\sim\D^m}[\widehat{\pl}^{\operatorname{boot},S}_{\varepsilon(m, \alpha)}(A) \geq 1 - \alpha - \gamma] \geq 1-\alpha - \exp(-\frac{6B\gamma^2}{4\gamma+3}).
\end{equation*}
In particular, we have that
\begin{equation*}
    \liminf_{m\rightarrow\infty}\Pr_{S\sim\D^m}\qty[\widehat{\pl}^{\operatorname{boot}, S}_{\varepsilon(m, \alpha-\gamma)}(A) \geq 1-\alpha] \geq 1-\alpha
\end{equation*}
if $B \geq (4\gamma + 3)\log(\gamma\inv)/(6\gamma^2)$.
\end{theorem}
\begin{proof}
First note that 
\[\widehat{\pl}^{\text{boot}, S}_{\varepsilon(m, \alpha)} \sim \frac{1}{B}\sum_{i=1}^B\operatorname{Bernoulli}(\pl^{\text{boot}, S}_{\varepsilon(m, \alpha)}).\] 
Hence, by Bernstein's Inequality, we have that
\begin{equation*}
    \Pr_{S\sim\D^m}[\widehat{\pl}^{\text{boot}, S}_{\varepsilon(m, \alpha)} \geq \pl^{\text{boot},S}_{\varepsilon(m, \alpha)}-\gamma] \geq 1 - \exp(-\frac{6B\gamma^2}{4\gamma+3}).
\end{equation*}
We know from \Cref{thm:plboot} that for large enough $m$,
\begin{equation*}
    \Pr_{S\sim\D^m}[\pl^{\text{boot},S}_{\varepsilon(m, \alpha)} \geq 1-\alpha] \geq 1-\alpha.
\end{equation*}
We thus have that
\begin{align*}
    \Pr[\widehat{\pl}^{\text{boot}, S}_{\varepsilon(m, \alpha)} 
    \geq 1-\alpha - \gamma] 
    &\geq \Pr[\widehat{\pl}^{\text{boot}, S}_{\varepsilon(m, \alpha)} \geq \pl^{\text{boot},S}_{\varepsilon(m, \alpha)}- \gamma \text{ and } \pl^{\text{boot},S}_{\varepsilon(m, \alpha)} \geq 1-\alpha] \\
    &\geq \Pr[\widehat{\pl}^{\text{boot}, S}_{\varepsilon(m, \alpha)} \geq \pl^{\text{boot},S}_{\varepsilon(m, \alpha)} - \gamma] + \Pr[\pl^{\text{boot},S}_{\varepsilon(m, \alpha)} \geq 1-\alpha] - 1\\
    &\geq 1 - \exp(-\frac{6B\gamma^2}{4\gamma+3}) + 1 - \alpha - 1 \\
    &=  1-\alpha - \exp(-\frac{6B\gamma^2}{4\gamma+3})
\end{align*}
as desired. We may then substitute $\alpha-\gamma$ for $\alpha$ to arrive at
\begin{equation*}
    \Pr[\widehat{\pl}^{\text{boot}, S}_{\varepsilon(m, \alpha - \gamma)} \geq 1-\alpha] \geq 1-\alpha + \gamma -\exp(-\frac{6B\gamma^2}{4\gamma+3}),
\end{equation*}
and the right hand side is at least $1-\alpha$ if $\gamma - \exp(-6B\gamma^2/(4\gamma+3)) \geq 0$, or equivalently if $B \geq (4\gamma+3)\log(\gamma\inv)/(6\gamma^2)$. 
\end{proof}

The practical consequence of \Cref{thm:samplebootplvalid} is as follows: When estimating plausibility by bootstrapping, the practitioner has to make a choice between the following:
\begin{enumerate}
    \item The confidence set $\Thetahat_S^{\varepsilon(m, \alpha)}$ keeps its usual tolerance (i.e. size). In exchange, rather than validity at level $\alpha$ with high probability, we have validity at level $\alpha + \gamma$ with a slightly smaller probability, where $\gamma$ decreases and the probability increases with the number of bootstrap samples. This is a reasonable choice to make when the number of bootstrap samples that can be taken is limited (e.g. due to computational reasons). If one were to this bootstrapping methodology to conduct hypothesis tests, one can minimize the type I error bound on the hypothesis test $H_0: \theta_0 \in A$ by selecting $\alpha$ and $\gamma$ appropriately for fixed $B$.
    
    \item We maintain validity at level $\alpha$ with high probability; thus, $\widehat{\pl}^{\text{boot}, S}_\varepsilon$ remains an asymptotic $p$-value. However, our confidence set $\Thetahat_S^{\varepsilon(m, \alpha-\gamma)}$, where the number of bootstrap samples necessary increases as $\gamma$ decreases, has a slightly larger tolerance and hence is less informative about the location of the risk minimizer. This is a reasonable choice to make when one can take as many bootstrap samples as desired. For hypothesis testing, doing so allows one to fix the type I error bound at $\alpha$ for the hypothesis test $H_0:\theta_0\in A$; the choice of $\gamma$ can impact the power of this test.
\end{enumerate}

\begin{example}\normalfont{
Recall from \Cref{ex:ber} the model for Bernoulli distributed data: $\X = \{\varnothing\}$, $\Y = \{0, 1\}$, $\H = \{x \mapsto \theta \setst \theta \in \{0, 1\}\}$, and $L(y, y') = |y - y'|$.

The only sets $A$ to estimate beliefs and plausibilities for are $\{0\}$, $\{1\}$, and $\{0, 1\}$. Note that the last set is the entire parameter space $\Theta$; since $\pl(\Theta) = 1$ necessarily, we focus on estimating the plausibility of the singleton sets. 
Recall that our theorem for validity of bootstrapping requires that $\Theta_0^\delta \subseteq A$. Now, the definition of $\Theta_0^\delta$ indicates that
\begin{equation*}
    \Theta_0^\delta = \{\theta \in\{0, 1\}\setst (1-p)^\theta p^{1-\theta} \leq \min\{p, 1-p\} + \delta\}
\end{equation*}
and it is straightforward to check that $0 \in \Theta_0^\delta$ if and only if $p < (1+\delta)/2$ and $1 \in \Theta_0^\delta$ if and only if $p > (1-\delta)/2$. That is,
\begin{equation*}
    \Theta_0^\delta \not\subseteq A \text{ iff } \frac{1-\delta}{2} < p < \frac{1 + \delta}{2}.
\end{equation*}
In particular, bootstrapping the singleton sets is not necessarily valid for $p = 1/2$. This makes intuitive sense, as the risk minimizer $\theta_0$ is the mode of the data-generating distribution, and when $p = 1/2$ neither singleton set can capture both modes. 

Now, our theorem indicates that bootstrapping is valid when the sample size is large enough that $\varepsilon(m, \alpha) \leq \delta$. Thus, we see that the closer $p$ is to $1/2$, the larger the sample size must be to ensure the validity of bootstrapping. 

To illustrate, we generate data from $\operatorname{Bernoulli}(0.499)$ at various sample sizes. We set the significance level to $\alpha = 0.05$ and also set $\gamma = \alpha/2$, setting the number of bootstrap samples as indicated by the theorem. For each generated sample, we checked whether or not $\widehat{\pl}_{\varepsilon(m, \alpha-\gamma)}^{\text{boot}, S}(\{0\}) \geq 1-\alpha$; we repeated this 1000 times for every sample size and report the frequency of this event. Results are shown in \Cref{fig:berbootstrap}.

Note that based on the above calculations, we require a sample size of 1,551,107 to guarantee validity at level $1-\alpha$ with probability at least $1-\alpha$. Evidently, our sample size requirement provided by the theorem is very conservative, as we appear to attain 95\% coverage nearly a full order of magnitude before 1.5 million.

It is also interesting to note that the coverage probability does not monotonically increase with the sample size $m$. This is because increasing the sample size has two competing effects: An increase in $m$ will decrease $\varepsilon(m, \alpha)$ and thus decrease the probability $\Thetahat_{S_i}^{\varepsilon(m, \alpha)}$ intersects with the set $A$ of interest (consequently decreasing the plausibility), but it will also allow the sample to be more representative of the population, allowing the estimated plausibilities to be more likely to attain the correct coverage of at least $1-\alpha$.

\begin{figure}[htbp]
    \centering
    \includegraphics[width=0.60\textwidth]{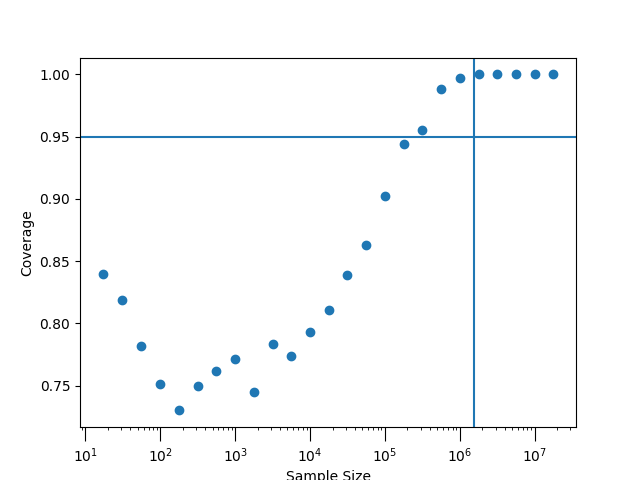}
    \caption{Empirical coverage over 1000 trials of the event $\widehat{\pl}^{\text{boot}, S}_{\varepsilon(m, \alpha - \gamma)}(\{0\}) \geq 1 - \alpha$ for $\alpha = 0.05$ and $\gamma = \alpha/2$. A vertical line is plotted at $m = 1,551,107$, where our theory guarantees 95\% coverage.}\label{fig:berbootstrap}
\end{figure}
}
\end{example}

\begin{example}\normalfont{
In \Cref{ex:neuralnet}, we mentioned that while a hypothesis test for the exact locations of the weights and biases of a neural network is not typically of great use, what is often of great importance in application is the appropriate level of sparsity of the parameters. Thus, in continuation of \Cref{ex:neuralnet}, consider a regularized neural network with 10 nodes in the only hidden layer using a sigmoid activation function, whose hypothesis class is given by
\begin{equation*}
    \H = \left\{\vb*{x}\mapsto \sigma\qty(\sum_{i=1}^{10} \vb*{w}^\top \sigma(\vb*{U}\vb*{x})) \altsetst \norm{\vb*{U}_i}_2 \leq M, \norm{\vb*{w}}_1 \leq \lambda \right\}
\end{equation*}
and is trained using the $L^1$ loss function. Then we have that
\begin{equation*}
    \varepsilon(m, \alpha) = \frac{8M\cdot\lambda}{m}\sqrt{\sum_{i=1}^m\norm{x_i}_2^2} + 6\sqrt{\frac{\log(4/\alpha)}{2m}}.
\end{equation*}

Let us suppose we are interested in the regularization parameter $\lambda$ for the weight vector $\vb*{w}$, where we limit $\lambda \in [0, 5]$. We may conduct a hypothesis test for the optimal regularization parameter $\lambda_0$ such as $H_0: \lambda_0\leq 0.25$. Suppose, furthermore, that the practitioner limits themselves to at most $B = 50$ bootstrap resamples. Then, according to \Cref{thm:samplebootplvalid}, the practitioner can reject $H_0$ at level $\alpha+\exp(-6B\gamma^2/(4\gamma+3))$ by rejecting if $\widehat{\pl}^{\text{boot},S}_{\varepsilon(m, \alpha)} < 1-\alpha-\gamma$ for some well-chosen $\alpha$ and $\gamma$ (supposing that the sample size is large enough). Optimization shows that the type I error rate of this test cannot be guaranteed to be less than roughly $0.242$ given the limit on the number of bootstrap iterations. 

To illustrate the hypothesis testing process, we can train this neural network to classify handwritten digits from the MNIST data set \citep{LeCun1998} as either even or odd. We test the aforementioned hypothesis at the best possible type I error rate, and find that for the MNIST data,  $\widehat{\pl}^{\text{boot},S}_{\varepsilon(m, \alpha)} = 0.400$. The critical value was $1-\alpha-\gamma \approx 0.585$, so we reject $H_0$ and (correctly) conclude that there is strong evidence (at significance level $0.242$) that the optimal regularization parameter is at least $0.25$. 
}    
\end{example}

\section{Comparison to Generalized Inferential Models}\label{sec:IMs}
An alternative approach to exploring notions of validity for ML, also based on imprecise probability, is given by generalized inferential models. 

The inferential models (IM) framework \citep{martin2013} is an approach to statistical inference that gives uncertainty quantification on unknown quantities $\theta_0$. This is done by assuming a data-generating mechanism $X = a(U, \theta_0)$, where the function $a$ and the distribution of the random variable $U$ are known, then creating a random set $\Thetahat_{S, E}$ for $\theta_0$ that depends on the observed sample $S$ as well as a random set $E$ intended to predict $U$. The distribution of the predictive random set is then described by the induced belief and plausibility functions:
\begin{align*}
    \bel_S(A) &= \Pr_E[\Thetahat_{S, E} \subseteq A \setst \Thetahat_{S,E} \neq \varnothing] \\
    \pl_S(A) &= \Pr_E[\Thetahat_{S,E} \cap A \neq \varnothing \given \Thetahat_{S,E} \neq \varnothing].
\end{align*}
For well-chosen predictive random sets, the induced belief and plausibility functions have a desirable validity property:
We say that an IM with belief/plausibility functions $\bel_S$ and $\pl_S$ is valid if for every measurable $A \subseteq 2^\Theta$ and every $\alpha \in (0, 1)$,
\begin{equation*}
    \sup_{\theta\in A} \Pr_{S\given\theta}[\pl_S(A) \leq \alpha] \leq \alpha.
\end{equation*}
Given a valid IM, the plausibility can be treated as a $p$-value---to test $H_0: \theta_0\in A$ against $H_1:\theta_0\not\in A$, we can reject if and only if $\pl_S(A) \leq \alpha$, and this test would have type I error rate at most $\alpha$.

The IM framework was recently extended to generalized IMs, or GIMs \citep{cella2022}, which do not need to make an assumption on the data generating mechanism and hence no longer rely on a predictive random set for the an auxiliary random variable. In this framework, an i.i.d. sample is generated from an unknown distribution that has some feature $\theta_0$ of interest. Then given a function $T_S$ which measures how well a given $\theta$ aligns with the observed sample $S$, the GIM gives an \textit{upper probability} for the assertion $\theta_0\in A$ as
\begin{equation*}
    \pl_S(A) = \sup_{\theta\in A}\{1 - G(T_S(\theta))\},
\end{equation*}
where $G$ is the cumulative distribution function for $T_S(\theta_0)$. \citet{cella2022} then shows that this is valid in the sense that if $\theta_0 \in A$, then
\begin{equation*}
    \Pr[\pl_S(A) \leq \alpha] \leq \alpha.
\end{equation*}
Because the distribution function $G$ is unknown, these plausibilities are not calculated directly, but rather through bootstrapping. In particular, \citet{cella2022} shows that if the consistency condition
\begin{equation*}
    \sup_{t\in\R}|G(t) - G^{\text{boot}}(t)| \overset{p}{\longrightarrow} 0
\end{equation*}
holds as the sample size $m$ goes to infinity (where $G^{\text{boot}}$ is the empirical cumulative distribution function of the bootstrapped values of $T_S(\thetahat)$), then
\begin{equation*}
    \limsup_{m\rightarrow\infty} \Pr[\pl^{\text{boot}}_S(A) \leq \alpha] \leq \alpha
\end{equation*}
if $\theta_0 \in A$. Thus, bootstrapped plausibilities from GIMs yield hypothesis tests for $H_0:\theta_0\in A$ with an asymptotically correct type I error rate.

In the context of machine learning, \citet{cella2022} suggests that the natural choice for $T_S$ is given by $T_S(\theta) = \rhat_S(\theta) - \inf_{\vartheta} \rhat_S(\vartheta)$. The bootstrapping approach then allows us to construct asymptotically correct confidence regions for the risk minimizer $\theta_0$ so long as the consistency condition is satisfied.

The GIM framework using the suggested $T_S$ function is clearly an alternative approach to exploring the validity of ML models. However, there are key differences between our approach and that of GIMs. Firstly, although both our approach and GIMs use bootstrapping to calculate plausibilities in practice to achieve prespecified type I error guarantees, GIMs only provide asymptotic validity, whereas we are able to arrive at validity at both finite sample sizes and at finite numbers of bootstrap resamples. Similarly, the confidence sets for $\theta_0$ in GIMs are obtained from inverting the bootstrap test, and so are only approximately valid for large sample sizes; our sets of $\varepsilon$-AERMs are valid for $\theta_0$ at all sample sizes and can be computed without bootstrapping. Furthermore, our approach gives clear bounds on the size of the confidence set and has theory regarding its consistency properties---such theory regarding the statistical power of the method is not currently present in the GIM approach. Finally, we note from an imprecise probability perspective that $\varepsilon$-plausibility in ML models is not random, whereas GIM plausibilities are; this yields different definitions for validity: GIMs require plausibilities of ``true" statements to be small with low probability, whereas we simply require that the plausibility be large.

\section{Concluding Remarks and Future Work}\label{sec:conclusion}
We have seen that for machine learning models with the uniform convergence property, it is possible to construct valid confidence sets for the model's risk minimizer $\theta_0$ via the set of $\varepsilon$-AERMs, despite the fact that these sets require essentially no structural assumptions regarding the data-generating distribution. Furthermore, if the practitioner is willing to make stronger assumptions on the model, this information can often be incorporated to construct more efficient confidence sets. We also demonstrated the necessity of the uniform convergence property by showing that for models that lack this property, it may be impossible to perform inference via ERMs at all.

Additionally, we saw that if the data-generating distribution is sufficiently known so that the distribution of the valid confidence set can be calculated, one can often determine whether or not $\theta_0 \in A$ by calculating the plausibility of the set $A$. Moreover, even when the data-generating distribution is completely unknown, one may still use bootstrapping in order to efficiently test hypotheses $H_0:\theta_0\in A$ at a given significance level $\alpha$ if the strong uniform convergence property holds. We finally note that while other methods can also yield (potentially stronger) inference on parameters of interest, such methods also typically have to make stronger, possibly unjustified assumptions about the model or rely on asymptotic theory in lieu of accomplishing finite-sample validity.

In future work, we plan to study to what extent this theory still applies to non-uniformly-learnable ML models. We also plan to investigate the statistical power of the hypothesis tests that use valid plausibilities so that we may determine how best to choose sample sizes and significance levels to attain a desired level of power. Additionally, our examples illustrate that these confidence sets and hypothesis tests can be overly-conservative; tighter bounds on uniform convergence functions for common machine learning models and data-generating scenarios (or methods to estimate these bounds from the data) must be further investigated in order to mitigate this phenomenon.

% Acknowledgements and Disclosure of Funding should go at the end, before appendices and references

%\acks{All acknowledgements go at the end of the paper before appendices and references. Moreover, you are required to declare funding (financial activities supporting the submitted work) and competing interests (related financial activities outside the submitted work). More information about this disclosure can be found on the JMLR website.}


\bibliography{manuscript}

\end{document}